\documentclass[a4paper]{llncs}
\usepackage{alltt}
\usepackage{amssymb}
\usepackage{caption}
\usepackage{paralist}
\usepackage{nicefrac}
\usepackage{booktabs}
\usepackage{xspace}
\usepackage{amsmath}
\usepackage{marvosym}
\usepackage{wasysym}
\usepackage{verbatim}
\usepackage{float}
\usepackage{hyperref}
\usepackage{multirow}
\usepackage{cite}
\usepackage[capitalize]{cleveref}
\usepackage[per-mode=symbol,detect-all]{siunitx}
\usepackage{graphics}
\usepackage{amssymb}
\usepackage{enumitem}
\usepackage[table]{xcolor}
\usepackage{booktabs}
\usepackage{colortbl}
\usepackage{ifthen}
\usepackage{mathdots}
\usepackage{arydshln}
\usepackage[outline]{contour}
\usepackage{graphicx}
\usepackage{wrapfig}
\usepackage{subfig}
\usepackage{algorithm}
\usepackage{algpseudocode}
\usepackage{MnSymbol}
\usepackage{placeins}

\definecolor{navy}{RGB}{0,0,128}

\usepackage{tikz,ifthen,pgfplots}
\usetikzlibrary{arrows,trees,backgrounds,automata,shapes,decorations,plotmarks,fit,calc,positioning,shadows,chains}
\usetikzlibrary{quotes,arrows.meta}
\tikzstyle{every pin edge}=[<-,shorten <=1pt]
\tikzstyle{neuron}=[circle,fill=black!25,minimum size=17pt,inner sep=0pt]
\tikzstyle{input neuron}=[neuron, fill=green!50]
\tikzstyle{output neuron}=[neuron, fill=red!50]
\tikzstyle{hidden neuron}=[neuron, fill=blue!50]
\tikzstyle{small neuron}        =[hidden neuron, draw, minimum size=15pt]
\tikzstyle{small input neuron}  =[input neuron , draw, minimum size=15pt]
\tikzstyle{small output neuron} =[output neuron, draw, minimum size=15pt]
\tikzstyle{annot} = [text width=4em, text centered]
\tikzstyle{nnedge} = [-{stealth},shorten >=0.1cm, shorten <=0.05cm,line width=0.8pt,black]
\tikzstyle{edge} = [->,line width = 0.3pt, shorten >=0.2cm]
\tikzstyle{edgeWide} = [->,line width = 2pt, , shorten >=0.2cm]
\usetikzlibrary{calc}

\tikzset{every picture/.style={line width=0.75pt}} 
\tikzstyle{BadSquare}=[rectangle,fill=red!30!white,minimum size=25pt,inner 
sep=0pt]
\tikzstyle{InitSquare}=[rectangle,fill=green!30!white,minimum size=25pt,inner 
sep=0pt]

\newcommand{\mysubsection}[1]{\medskip\noindent\textbf{#1}}

\newcommand{\relu}{\text{ReLU}\xspace}
\newcommand{\kisoneappendix}{\text{k=1}\xspace}
\newcommand{\kismaxkappendix}{\text{k=maxk}\xspace}
\newcommand{\kmhskcxp}{\text{K-MHS(k-Cxps)}\xspace}
\newcommand{\verifysingleton}{\text{Verify((F$\setminus$\{f\}=v,N,$Q_{\neg
			c}$)}\xspace}
\newcommand{\verifyexplanation}{\text{Verify((Explanation$\setminus$\{f\})=v,N,$Q_{\neg
			c}$)}\xspace}

\newcommand{\verifybinary}{\text{Verify((Explanation$\setminus$\{L,L$+1$,...,Mid\})=v,N,$Q_{\neg
			c}$)}\xspace}

\newcommand{\verifylocalsingletons}{\text{Verify(Explanation$\setminus \{f^\prime\}=C$,N,$Q_{\neg
			c}$)}\xspace}

\newcommand{\verifypair}{\text{Verify((F$\setminus$\{a,b\}=v,N,$Q_{\neg
			c}$)}\xspace}

\newcommand{\singletons}{\text{Singletons}\xspace}
\newcommand{\bsingletons}{\text{BSingletons}\xspace}
\newcommand{\bpairs}{\text{BPairs}\xspace}
\newcommand{\bsingletonsinstance}{\text{BS}\xspace}
\newcommand{\localsingletons}{\text{LocalSingletons}\xspace}

\newcommand{\explanation}{\text{Explanation}\xspace}
\newcommand{\featuresleft}{\text{RemainingFeatures}\xspace}

\newcommand{\tub}{\text{$T_{\text{UB}\xspace}$}\xspace}
\newcommand{\tlb}{\text{$T_{\text{LB}\xspace}$}\xspace}
\newcommand{\pairs}{\text{Pairs}\xspace}
\newcommand{\allpairs}{\text{AllPairs}\xspace}
\newcommand{\allpairsinput}{\text{Distinct pairs of F$\setminus$Singletons}\xspace}
\newcommand{\singlepair}{\text{(a,b)}\xspace}
\newcommand{\free}{\text{Free}\xspace}
\newcommand{\freewithf}{\text{Free$\ \cup\ $\{f\}}\xspace}
\newcommand{\upperb}{\text{UB}\xspace}
\newcommand{\lowerb}{\text{LB}\xspace}
\newcommand{\lowerbbundle}{\text{$LB_{Bundle}$}\xspace}
\newcommand{\allbundlesingletons}{\text{$\cup$BSingletons}\xspace}
\newcommand{\mwvcbpairs}{\text{MWVC(BPairs)}\xspace}
\newcommand{\mhsbbcxps}{\text{MHSB($Cxps_B$)}\xspace}
\newcommand{\bcxps}{\text{$Cxps_B$}\xspace}
\newcommand{\minimumbundle}{\text{$E_{\text{MB}\xspace}$}\xspace}
\newcommand{\mvc}{\text{MVC(Pairs)}\xspace}
\newcommand{\mvcbundle}{\text{MVC(BPairs)}\xspace}

\newcommand{\mhscxp}{\text{MHS(Cxps)}\xspace}
\newcommand{\mhsbpairs}{\text{MHS($\cup$Bpairs)}\xspace}
\newcommand{\mhsbsingletons}{\text{MHS(BS)}\xspace}
\newcommand{\minimumexplanation}{\text{$E_M$}\xspace}

\newcommand{\sat}{\texttt{SAT}\xspace}
\newcommand{\unsat}{\texttt{UNSAT}\xspace}

\usepackage{bussproofs}
\usepackage{gensymb}

\newif\ifcomments
\commentstrue

\newif\ifoutline
\outlinefalse

\newif\iflong
\longtrue

\renewcommand{\paragraph}[1]{\vspace{1mm}\noindent{\bf #1}\ }

\begin{document}
	
	\title{Towards Formal XAI: Formally Approximate Minimal Explanations of
		Neural Networks}
	
	\author{Shahaf Bassan \and Guy Katz}
	\institute{The Hebrew University of Jerusalem, Jerusalem,
		Israel\\
		\email{ \{shahaf.bassan, g.katz\}@mail.huji.ac.il}
	}
	
	\maketitle
	
	\begin{abstract}
		With the rapid growth of machine learning, deep neural networks
		(DNNs) are now being used in numerous domains.  Unfortunately, DNNs
		are ``black-boxes'', and cannot be interpreted by humans,
		which is a substantial concern in safety-critical systems. To
		mitigate this issue, researchers have begun working on explainable AI
		(XAI) methods, which can identify a subset of input features that
		are the cause of a DNN's decision for a given input.  Most existing
		techniques are heuristic, and cannot guarantee the correctness of
		the explanation provided. In contrast, recent and exciting attempts have shown that
		formal methods can be used to generate provably correct
		explanations. Although these methods are sound, the computational
		complexity of the underlying verification problem limits their
		scalability; and the explanations they produce might sometimes be
		overly complex. Here, we propose a novel approach to tackle 
		these limitations. We
		\begin{inparaenum}[(i)]
			\item suggest an efficient, verification-based method for finding
			\emph{minimal explanations}, which constitute a \emph{provable
				approximation} of the global, minimum explanation;
			\item show how DNN verification can assist in calculating lower and
			upper bounds on the optimal explanation;
			\item propose heuristics that significantly improve the scalability
			of the verification process; and
			\item suggest the use of \emph{bundles}, which allows us to arrive at
			more succinct and interpretable explanations.
		\end{inparaenum}
		Our evaluation shows that our approach significantly
		outperforms state-of-the-art techniques, and produces
		explanations that are more useful to humans.  We thus regard this
		work as a step toward leveraging verification technology in
		producing DNNs that are more reliable and comprehensible.
	\end{abstract}
	
	\section{Introduction}
	\label{sec:Introduction}
	Machine learning (ML) is a rapidly growing field with a wide range of
	applications, including safety-critical, high-risk systems in the
	fields of health care~\cite{esteva2019guide},
	aviation~\cite{julian2019deep} and autonomous
	driving~\cite{bojarski2016end}. Despite their success, ML models, and
	especially deep neural networks (DNNs), remain ``black-boxes''
	--- they are incomprehensible to humans and are prone to unexpected
	behaviour and errors. This issue can result in major
	catastrophes~\cite{zhou2019metamorphic, staff2019case}, and also in
	poor decision-making due to brittleness or
	bias~\cite{angwin2016machine, goodfellow2014explaining}.

	In order to render DNNs more comprehensible to humans, researchers
	have been working on \emph{explainable AI} (\emph{XAI}), where we seek
	to construct models for explaining and interpreting the decisions of
	DNNs~\cite{ribeiro2018anchors, lundberg2017unified, selvaraju2017grad,
		ribeiro2016should}. Work to date has focused on heuristic
	approaches, which provide explanations, but do not provide
	guarantees about the correctness or succinctness of these
	explanations~\cite{camburu2019can, ignatiev2019validating,
		lakkaraju2020fool}. Although these approaches are an
	important step, their limitations might result in skewed
	results, possibly failing to meet the regulatory
	guidelines of institutions and organizations such as the European Union, the US
	government, and the OECD~\cite{marques2022delivering}. Thus, producing
	DNN explanations that are provably accurate remains of utmost
	importance.
	
	More recently, the formal verification community has proposed
	approaches for providing formal and rigorous explanations for
	DNN decision making~\cite{ignatiev2020towards,
		marques2022delivering, ignatiev2019abduction,
		shih2018symbolic}. Many of these approaches rely on the
	recent and rapid developments in DNN
	verification~\cite{akintunde2019verification, avni2019run,
		baluta2019quantitative, katz2017reluplex}.  These
	approaches typically produce an
	\emph{abductive explanation} (also known as a \emph{prime
		implicant}, or
	\emph{PI-explanation})~\cite{shih2018symbolic,
		ignatiev2019abduction, shi2020tractable}: a minimum subset
	of input features, which by themselves already determine the
	classification produced by the DNN, regardless of any other
	input features. These explanations afford formal guarantees,
	and can be computed via DNN
	verification~\cite{ignatiev2019abduction}.
	
	Abductive explanations are highly useful, but there are two major
	difficulties in computing them. First, there is the issue of
	scalability: computing locally minimal explanations
	might require a polynomial number of costly invocations of the
	underlying DNN verifier, and computing a globally minimal
	explanation is even more challenging~\cite{liberatore2005redundancy, barcelo2020model, ignatiev2019abduction}.
	The second difficulty is that users may sometimes prefer
	``high-level'' explanations, not based solely on input
	features, as these may be easier to grasp and interpret
	compared to ``low-level'', complex, feature-based
	explanations. 
	
	To tackle the first difficulty, we propose here new
	approaches for more efficiently producing verification-based
	abductive explanations. More concretely, we propose a method for
	\emph{provably approximating} minimum explanations, allowing stakeholders
	to use slightly larger explanations that can be discovered much more
	quickly. To accomplish this, we leverage the recently discovered dual
	relationship between explanations and contrastive
	examples~\cite{ignatiev2020contrastive}; and also take advantage of
	the sensitivity of DNNs to small adversarial
	perturbations~\cite{su2019one}, to compute both lower and
	upper bounds for the minimum explanation. In addition, we propose
	novel heuristics for significantly expediting the underlying
	verification process. 
	
	In addressing the second difficulty, i.e.~the interpretability
	limitations of ``low-level'' explanations, we propose to
	construct explanations in terms of \emph{bundles}, which are
	sets of related features. We empirically show that using our
	method to produce bundle explanations can significantly
	improve the interpretability of the results, and even the
	scalability of the approach, while still maintaining the
	soundness of the resulting explanations.

    To summarize, our contributions include the following: 
    \begin{inparaenum}[(i)]
    \item We are the first to suggest a method that formally produces sound and minimal abductive explanations that \emph{provably approximate} the global-minimum explanation.
    \item Our three suggested novel heuristics expedite the search for
      minimal abductive explanations, significantly outperforming the
      state of the art.
    \item We suggest a novel approach for using bundles to efficiently produce sound and provable explanations that are more interpretable and succinct.
	\end{inparaenum} 
	
	
	For evaluation purposes, we implemented our approach as a
	proof-of-concept tool. 
	Although our method can be
	applied to any ML model, we focused here on DNNs, where the
	verification process is known to be
	NP-complete~\cite{katz2017reluplex}, and the scalable
	generation of explanations
	is known to be challenging~\cite{ignatiev2019abduction, shi2020tractable}. We 
	used our tool to test the approach on DNNs trained for digit and clothing classification, and also compared
	it to state-of-the-art approaches~\cite{ignatiev2019abduction,
		ignatiev2019validating}. Our results indicate that our
	approach was successful in quickly producing meaningful
	explanations, often running  40\% faster
than existing tools.
        We
	believe that these promising results showcase the potential of
	this line of work.
	
	The rest of the paper is organized as follows. Sec.~\ref{sec:background} contains background on DNNs and their
	verification, as well as on formal, minimal
	explanations. Sec.~\ref{sec:approximations} covers the main
	method for calculating approximations of minimum explanations,
	and Sec.~\ref{heuristics_local_search} covers methods for
	improving the efficiency of calculating these
	approximations. Sec.~\ref{sec:bundles} covers the use of
	\emph{bundles} in constructing ``high-level'', provable
	explanations. Next, we
	present our evaluation in
	Sec.~\ref{sec:Evaluation}. Related work is covered in
	Sec.~\ref{sec:RelatedWork}, and we conclude in
	Sec.~\ref{sec:Conclusion}.	
	\section{Background}
	\label{sec:background}

	
	\mysubsection{DNNs.}  A deep neural network (DNN)~\cite{lecun2015deep}
	is a directed graph composed of layers of nodes, commonly called
	\emph{neurons}. In feed-forward NNs the data flows from the first
	(\emph{input}) layer, through intermediate (\emph{hidden}) layers, and
	onto an \emph{output} layer. A DNN's output is calculated by assigning
	values to its input neurons, and then iteratively calculating the
	values of neurons in subsequent layers. In the case of
	\emph{classification}, which is the focus of this paper, each output
	neuron corresponds to a specific \emph{class}, and the output
	neuron with the highest value corresponds to the class the input is
	classified to.
	
	\begin{wrapfigure}{r}{0.45\textwidth}
		
		\vspace{-1.2cm}
		\begin{center}
			\includegraphics[width=0.45\textwidth]{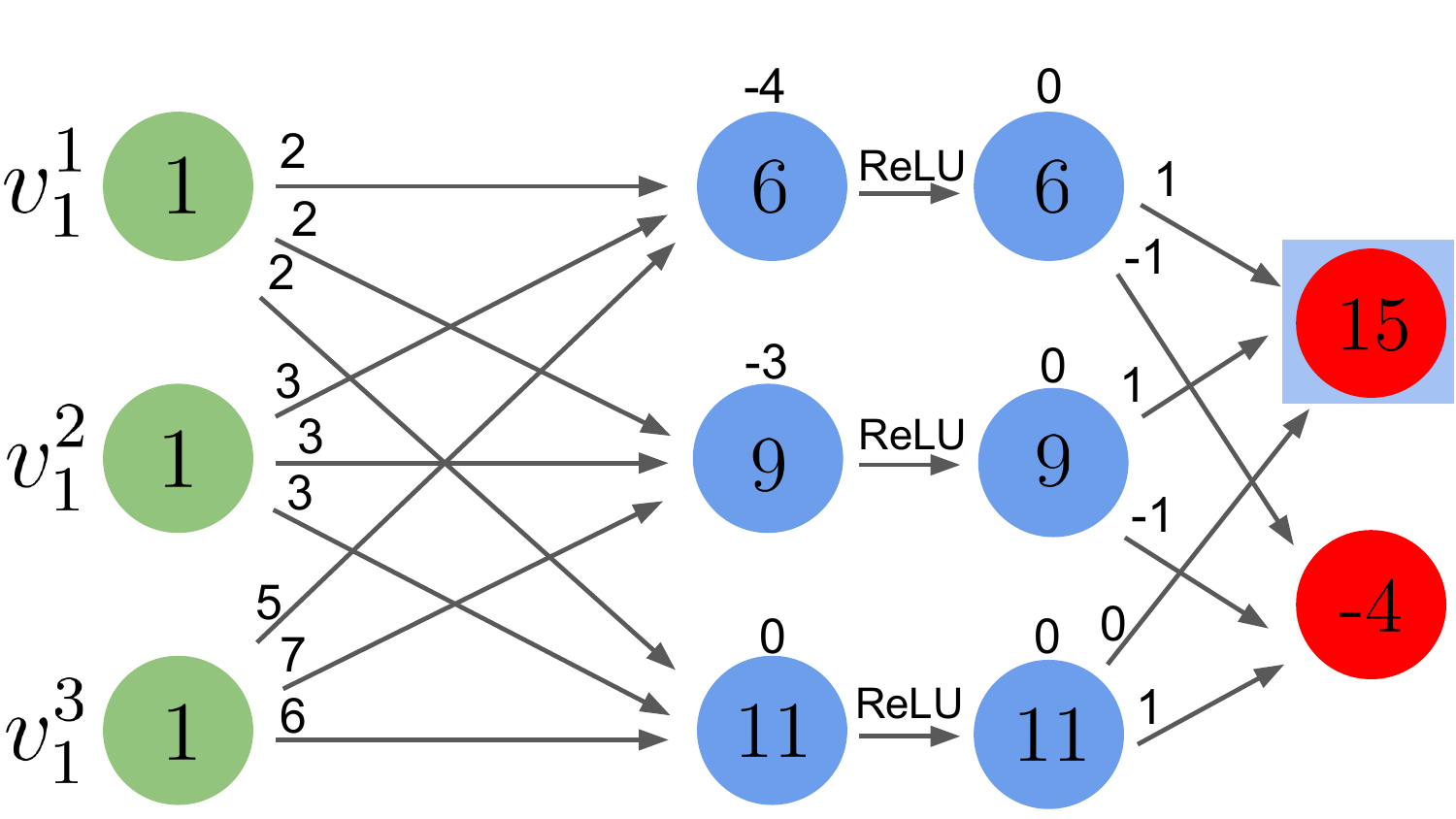}
			\caption{A simple DNN.}
			\label{fig:neural_network_example_1}
		\end{center}
		\vspace{-1cm}  
	\end{wrapfigure}
	
	Fig.~\ref{fig:neural_network_example_1} depicts a simple, feed-forward
	DNN. The input layer includes three neurons, followed by a weighted
	sum layer, which calculates an affine transformation of values from
	the input layer. Given the input
	$V_1=[1,1,1]^T$, the second layers computes the values
	$V_2=[6,9,11]^T$. Next comes a \relu{} layer, which computes the
	function $\relu(x)=\max(0,x)$ for each neuron in the preceding layer,
	resulting in $V_3=[6,9,11]^T$. The final (output) layer then computes
	an affine transformation, resulting in $V_4=[15,-4]^T$. This indicates
	that input $V_1=[1,1,1]^T$ is classified as the category corresponding
	to the first output neuron, which is assigned the greater value. 

	\mysubsection{DNN Verification.}  A DNN verification query is a tuple
	$\langle P, N, Q\rangle$, where $N$ is a DNN that maps an input vector
	$x$ to an output vector $y=N(x)$, $P$ is a predicate on $x$, and $Q$ is a
	predicate on $y$. A DNN verifier needs to decide whether there exists
	an input $x_0$ that satisfies $P(x_0) \wedge Q(N(x_0))$ (the \sat{}
	case) or not (the \unsat{} case). Typically, $P$ and $Q$ are expressed
	in the logic of real arithmetic~\cite{LiArLaBaKo20}. The DNN verification
	problem is known to be NP-Complete~\cite{katz2017reluplex}.
	
	\mysubsection{Formal Explanations.}  We focus here on
	explanations for classification problems, where a model is trained to
	predict a label for each given input.  A classification
	problem is a tuple $\langle F, D, K, N\rangle$ where
	\begin{inparaenum}[(i)]
		\item $F=\{1,...,m\}$ denotes the features;
		\item $D=\{D_1,D_2...,D_m\}$ denotes the domains of each of the
		features, i.e.~the possible values that each feature can take. The
		entire feature (input) space is hence
		$\mathbb{F}={D_1 \times D_2 \times...\times D_m}$;
		\item $K=\{c_1,c_2,...,c_n\}$ is a set of classes, i.e.~the possible
		labels; and
		\item $N:F\to K$ is a (non-constant) classification
                  function (in our case, a neural network).
	\end{inparaenum}
	A classification instance is the pair $(v,c)$, where $v\in \mathbb{F}$,
	$c\in K$, and $c=N(v)$. In other words, $v$ is mapped by the
	neural network $N$ to class $c$.
	
	Looking at $(v,c)$, we often wish to know why $v$ was
	classified as $c$. Informally, an \emph{explanation} is a
	subset of features $E\subseteq F$, such that assigning these
	features to the values assigned to them in $v$ already
	determines that the input will be classified as $c$,
	regardless of the remaining features $F\setminus E$. In other
	words, even if the values that are \emph{not} in the
	explanation are changed arbitrarily, the classification
	remains the same. More formally, given input
	$v=(v_1,...v_m)\in \mathbb{F}$ with the classification 
	$N(v)=c$, an explanation (sometimes referred to as an
	\emph{abductive explanation}, or an \emph{AXP}) is a subset of
	the features $E\subseteq F$, such that:
	\begin{equation}
		\label{eq:explanation}
		\forall(x\in \mathbb{F}).\quad [\bigwedge_{i\in E}(x_{i}=v_{i})\to(N(x)=c)]
	\end{equation}
	
	We continue with the running example from
	Fig.~\ref{fig:neural_network_example_1}.  For simplicity, we
	assume that each input neuron can only be assigned the values
        0 or 1. It can
	be observed that for input $V_1=[1,1,1]^T$, the set
	$\{ v_1^1, v_1^2 \}$ is an explanation; indeed, once the first
	two entries in $V_1$ are set to $1$, the classification
	remains the same for any value of the third entry (see
	Fig.~\ref{fig:neural_network_explanation}).  We can prove
	this by encoding a verification query
	$\langle P,N,Q\rangle = \langle E=v,N,Q_{\neg
		c}\rangle$, where $E$ is
	the candidate explanation, and $E=v$ means that we restrict
	the features in $E$ to their values in $v$; and $Q_{\neg
		c}$
	implies that the classification is not $c$.  An \unsat{} result
	for this query indicates that $E$ is an explanation for
	instance $(v,c)$.
	\begin{figure}
		\centering
		{\includegraphics[width=0.45\textwidth]{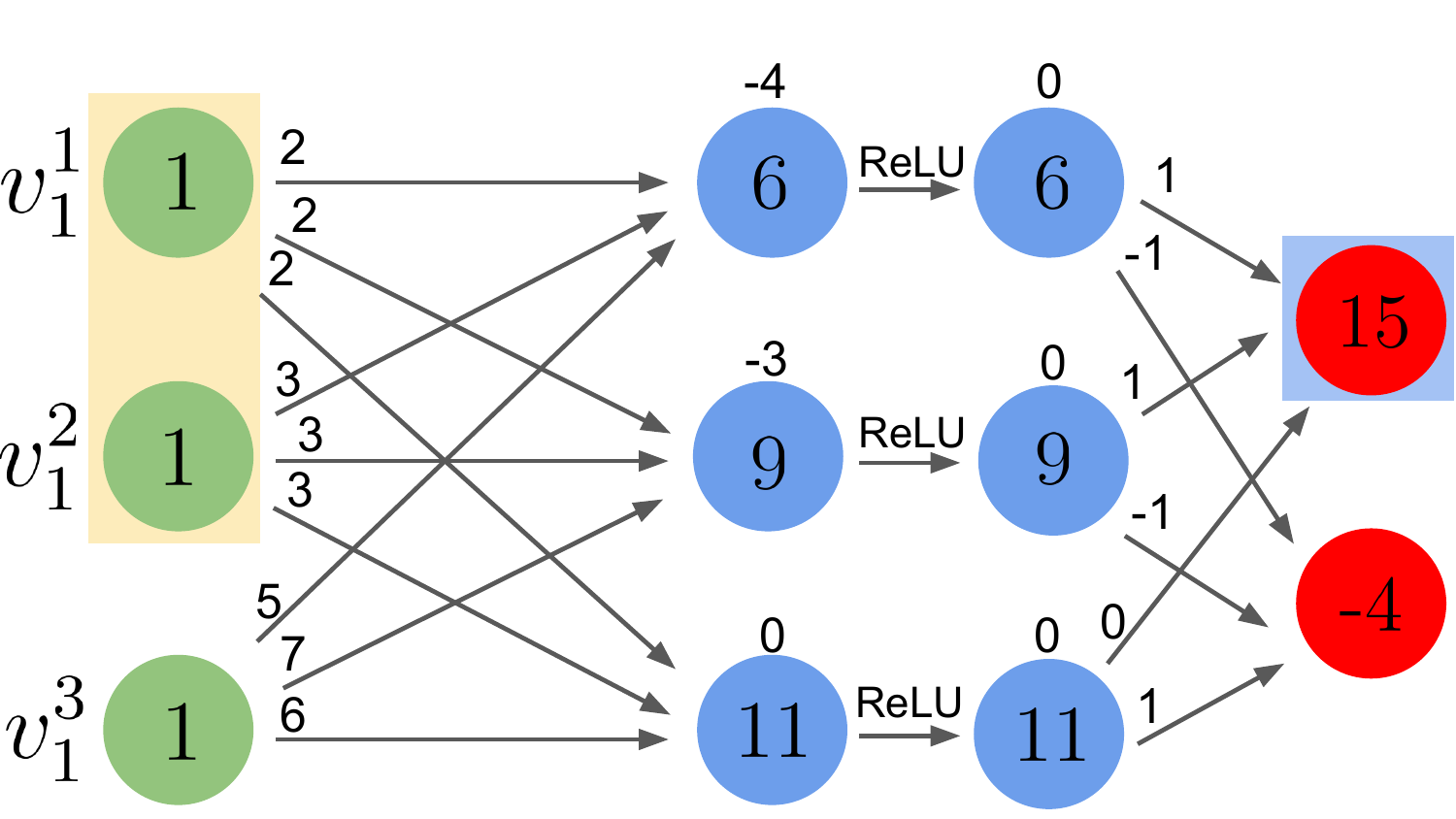}}
		\hfill
		{\includegraphics[width=0.45\textwidth]{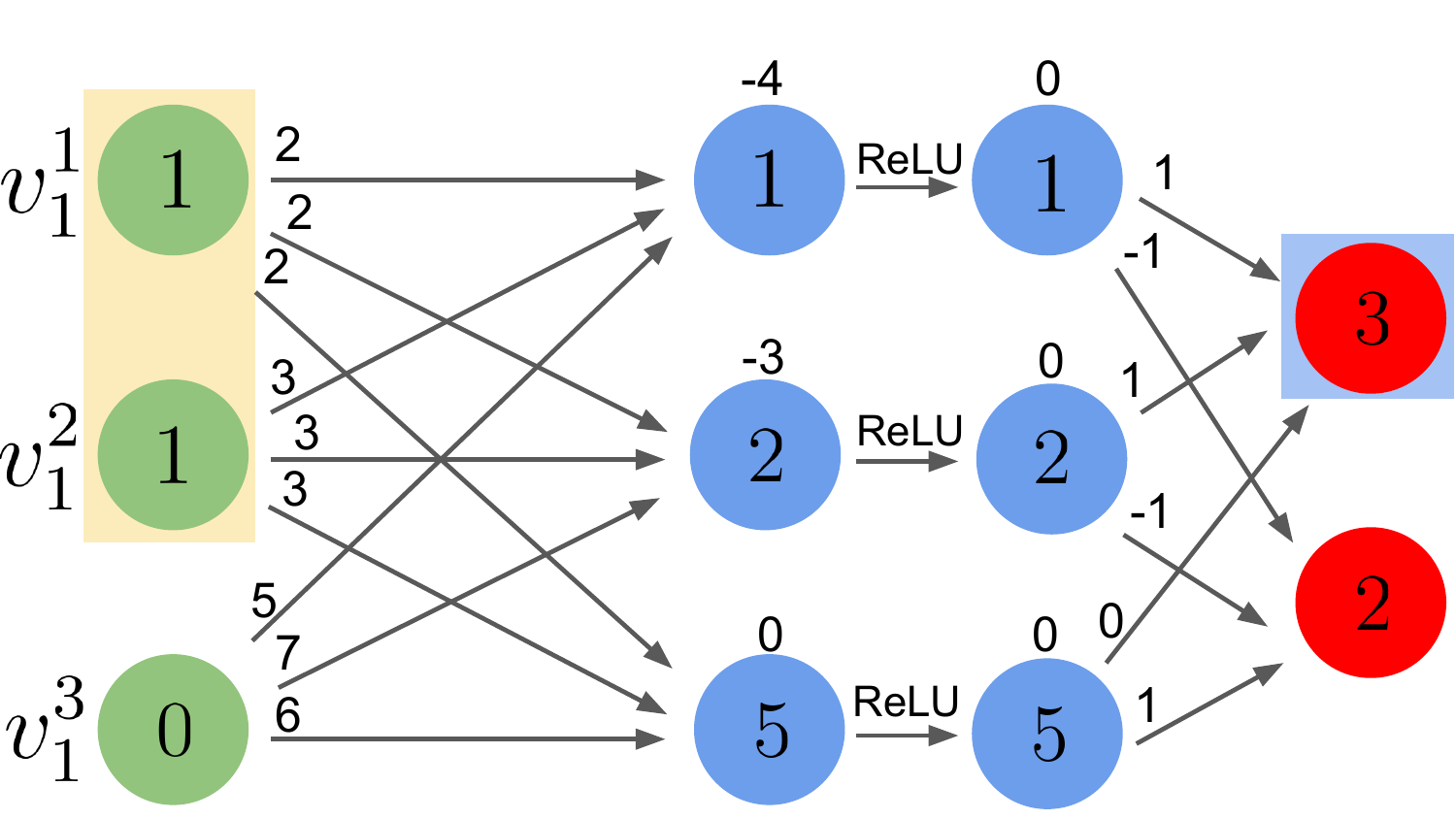}}
		\caption{$\{ v_1^1, v_1^2 \}$ is  an explanation for input $V_1=[1,1,1]^T$}
		\label{fig:neural_network_explanation}
	\end{figure}
	
	Clearly, the set of all features constitutes a trivial
	explanation. However, we are interested in \emph{smaller}
	explanation subsets, which can provide useful information
	regarding the decision of the classifier. More precisely, we search
	for \emph{minimal explanations} and \emph{minimum explanations}.  A
	subset $E\subseteq F$ is a \emph{minimal explanation} (also referred
	to as a \emph{local-minimal explanation}, or a \emph{subset-minimal
		explanation}) of instance $(v,c)$ if it is an explanation that ceases to be an
	explanation if even a single feature is removed from it: 
	\begin{equation}
		\begin{aligned}
			&(\forall(x\in \mathbb{F}).[\wedge_{i\in
				E}(x_{i}=v_{i})\to(N(x)=c)]) \wedge\\
			&(\forall(j\in E).[  \exists(y\in \mathbb{F}).[\wedge_{i\in E\setminus j}(y_{i}=v_{i})\wedge(N(y)\neq c)])
		\end{aligned}
	\end{equation}
	Fig.~\ref{fig:neural_network_minimal_explanation} demonstrates that
	$\{ v_1^1, v_1^2 \}$ is a minimal explanation in our running example:
	removing any of its features allows mis-classification.
	
	\begin{figure}	
		\centering
		{\includegraphics[width=0.45\textwidth]{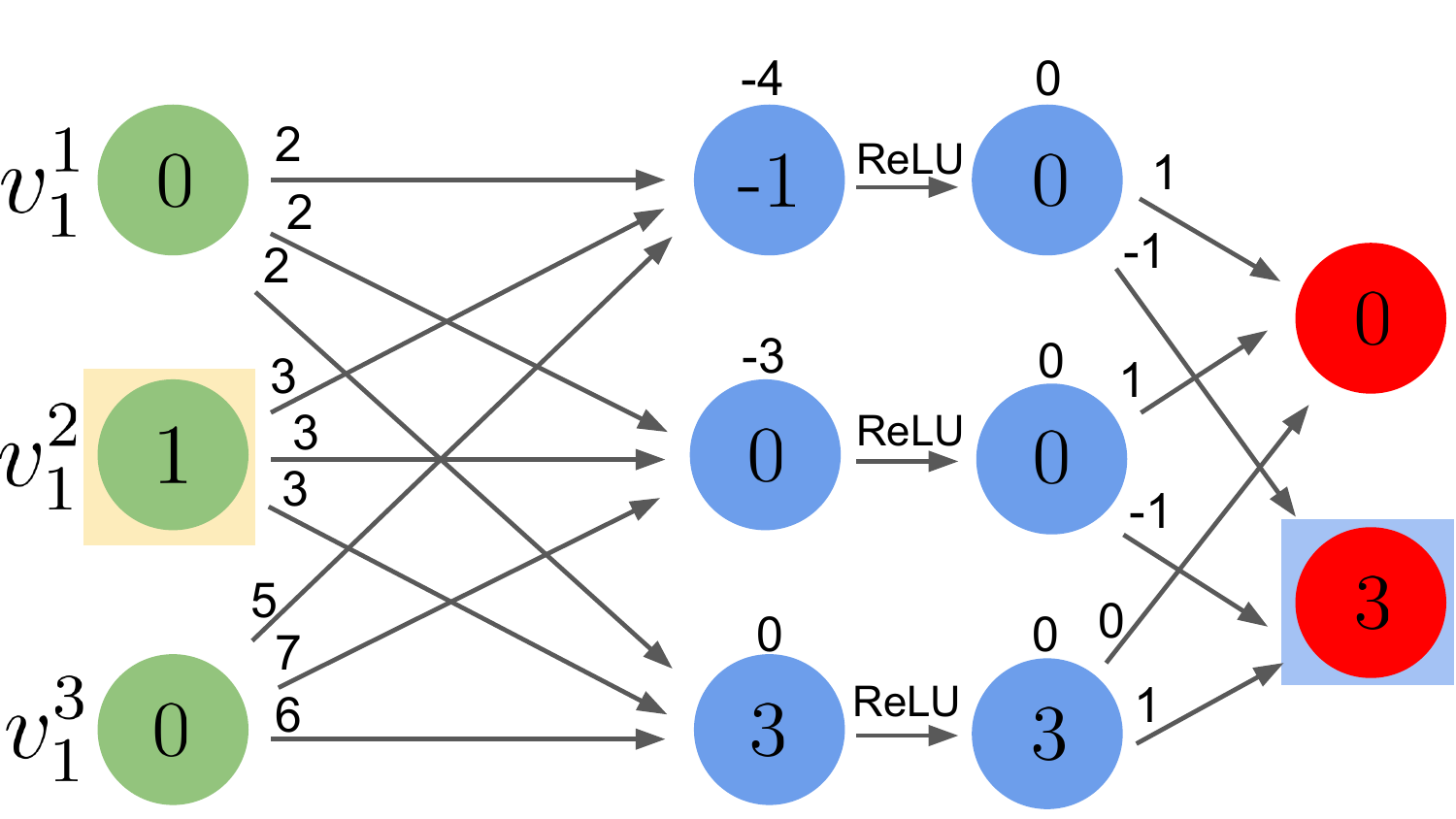}}
		\hfill
		{\includegraphics[width=0.45\textwidth]{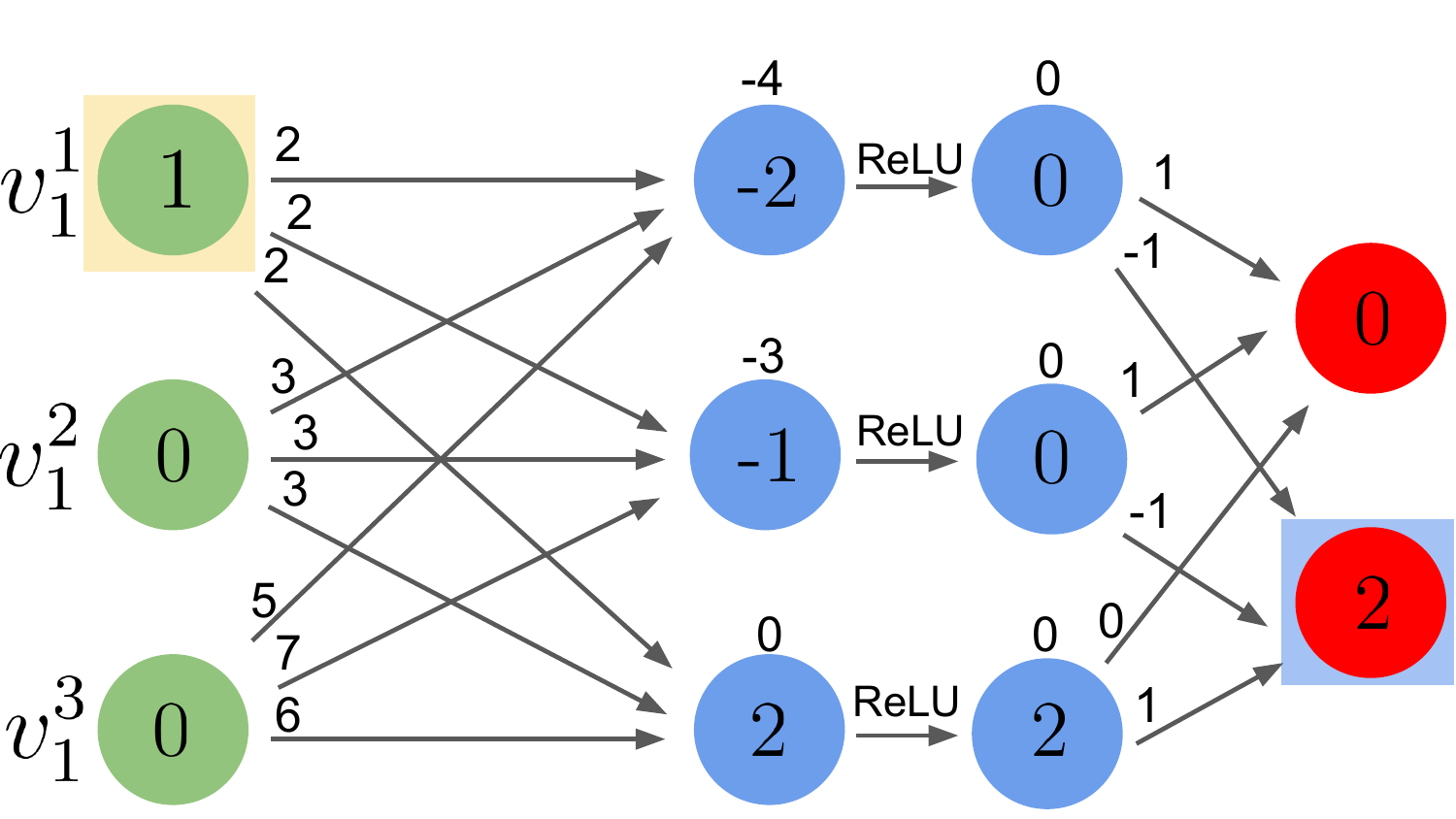}}
		\caption{$\{ v_1^1, v_1^2 \}$ is a minimal explanation for input $V_1=[1,1,1]^T$.}
		\label{fig:neural_network_minimal_explanation}
	\end{figure}
	
	A \emph{minimum explanation} (sometimes referred to as a \emph{cardinal
		minimal explanation} or a \emph{PI-explanation}) is defined as a
	minimal explanation of minimum size; i.e., if $E$ is a minimum
	explanation, then there does not exist a minimal explanation
	$E' \neq E$ such that $|E'|<|E|$.
	Fig.~\ref{fig:neural_network_minimum_explanation} demonstrates that $\{
	v_1^3 \} $ is a minimum explanation for our running example.
	
	\begin{figure}
		\centering
		{\includegraphics[width=0.325\textwidth]{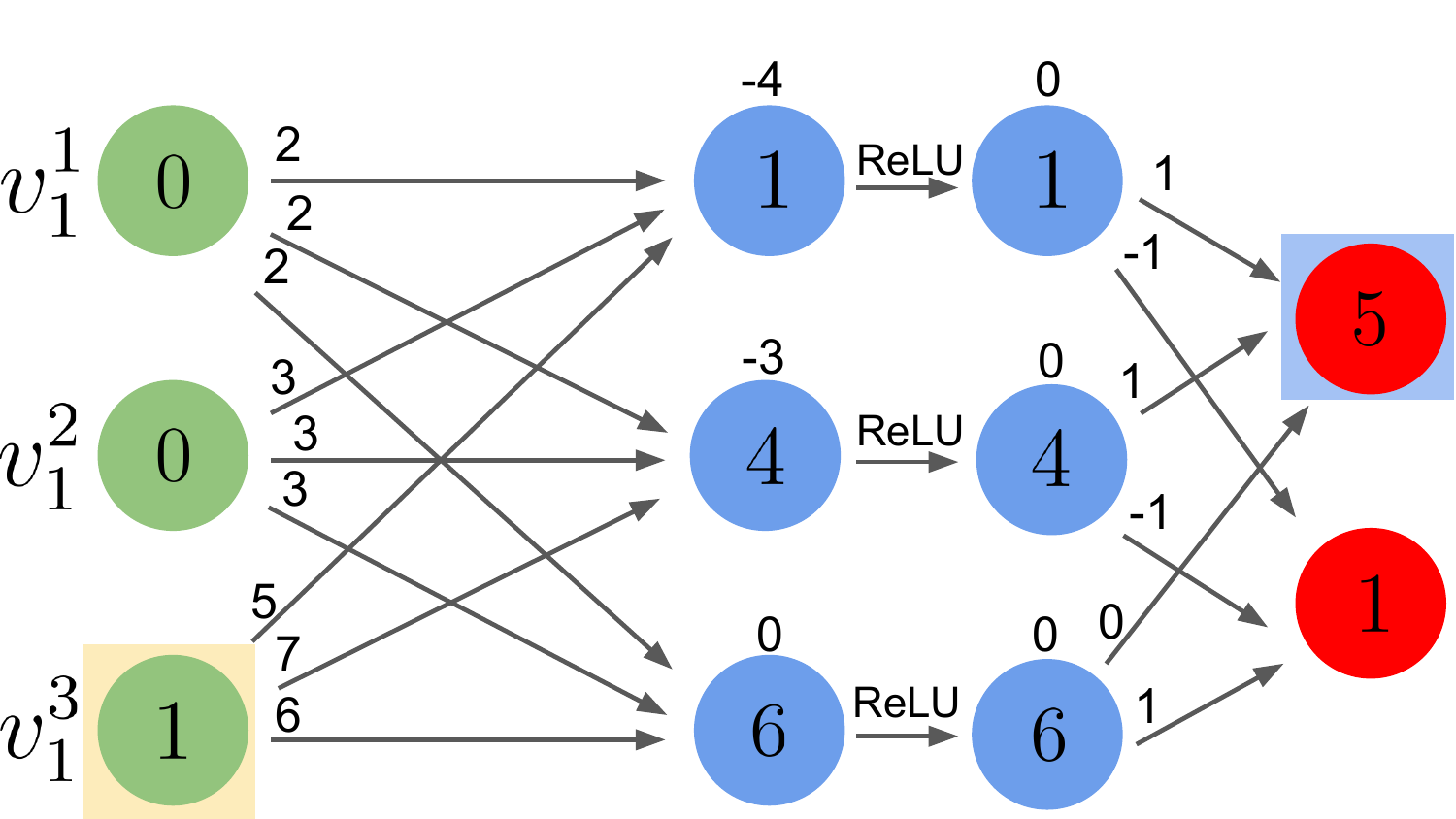}}
		\hfill
		{\includegraphics[width=0.325\textwidth]{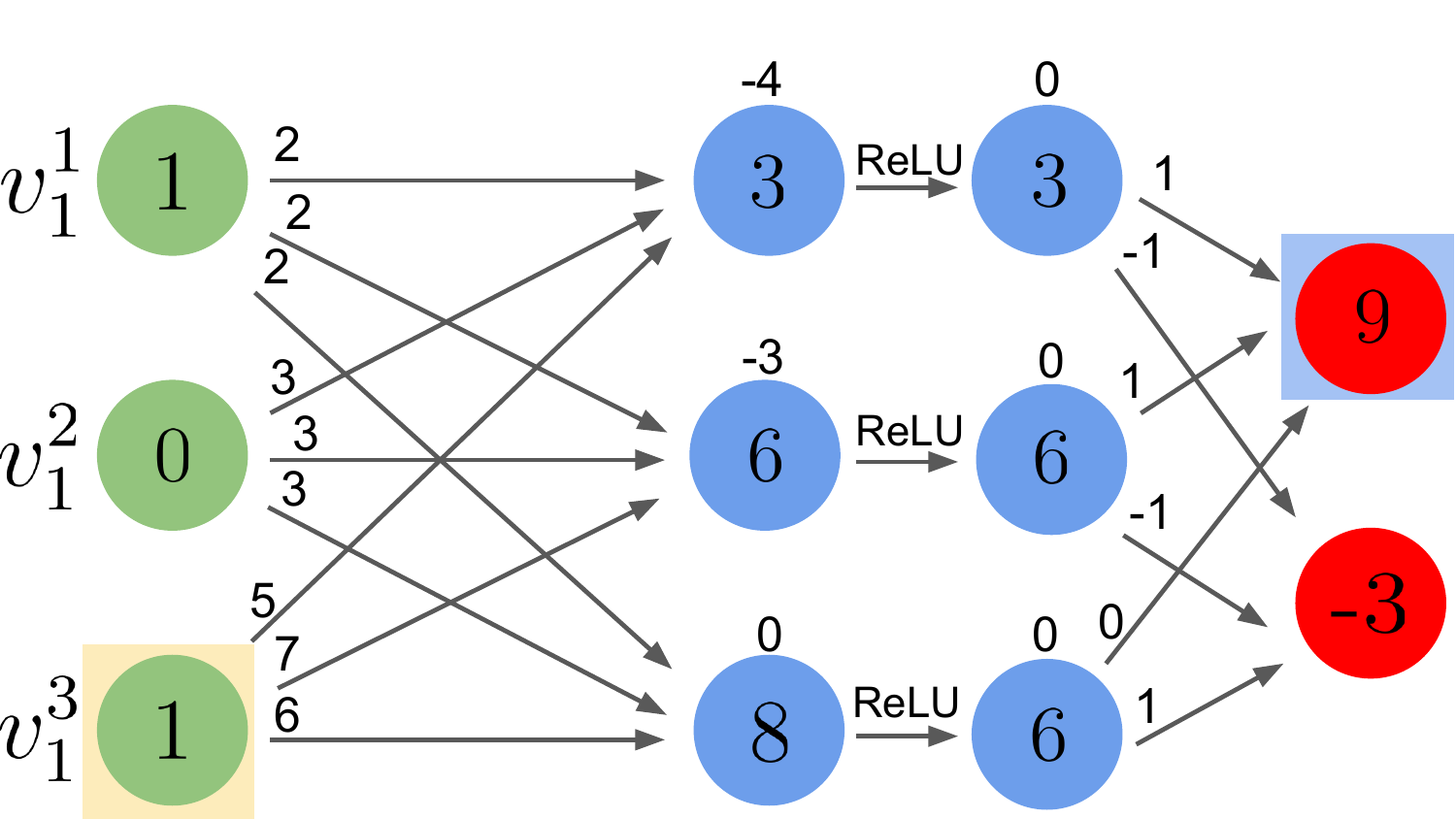}}
		\hfill
		{\includegraphics[width=0.325\textwidth]{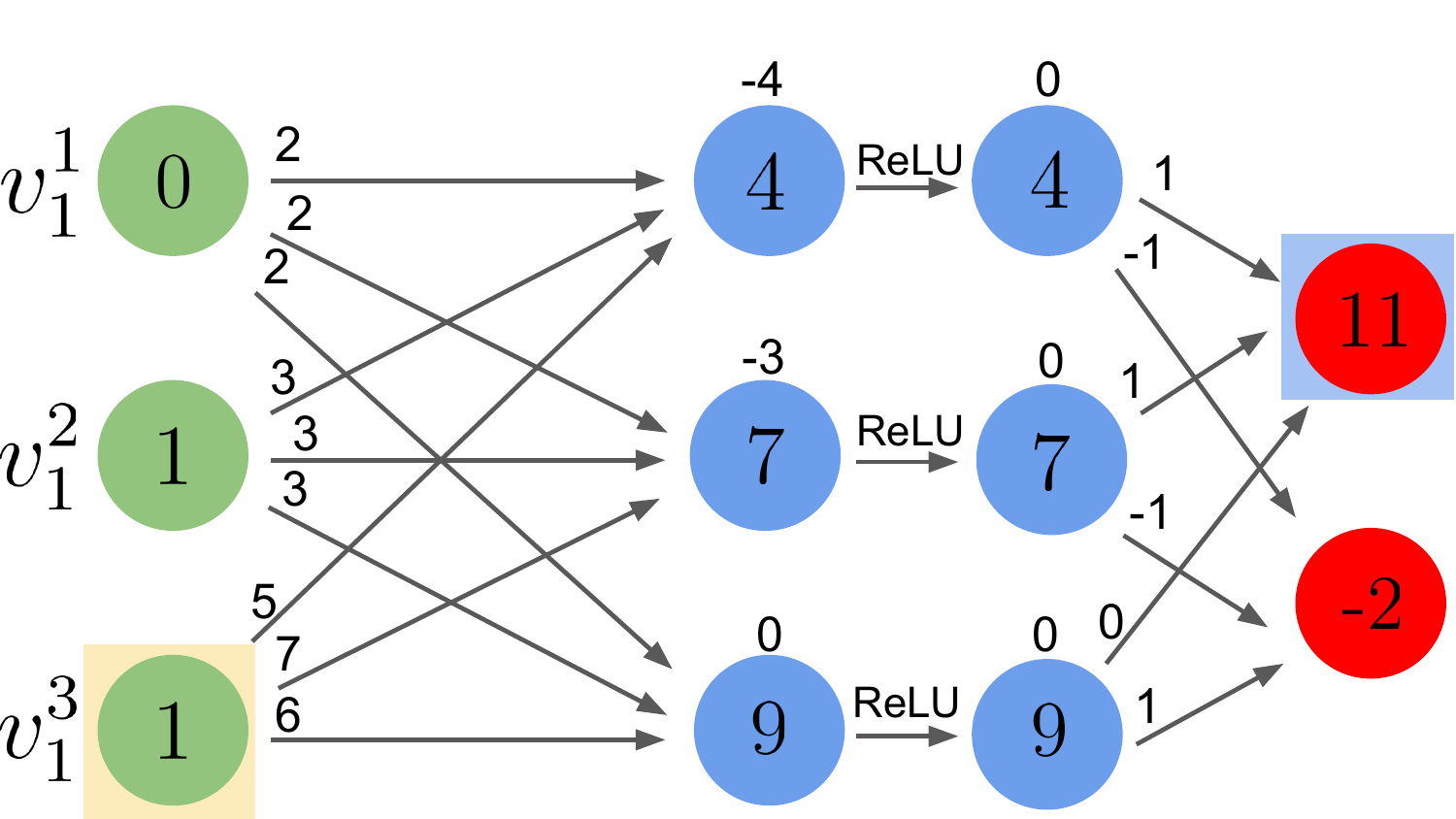}}
		\caption{$\{ v_1^3 \} $ is a minimum explanation for input $V_1=[1,1,1]^T$.}
		\label{fig:neural_network_minimum_explanation}
	\end{figure}
	
	\mysubsection{Contrastive Example.}
	\label{sec:contrastive-examples}
	A subset of features $C\subseteq F$ is called a \emph{contrastive
		example} or a \emph{contrastive explanation (CXP)} if
	altering the features in $C$ is sufficient to cause the
	misclassification of a given classification instance $(v,c)$:
	\begin{equation}
		\label{eq:contrastive_examples}
		\exists(x\in \mathbb{F}).[\wedge_{i\in F\setminus C}(x_{i}=v_{i})\wedge(N(x)\neq c)]
	\end{equation}
	\begin{wrapfigure}{r}{0.45\textwidth}
		\vspace{-0.8cm}
		\begin{center}
			\includegraphics[width=0.45\textwidth]{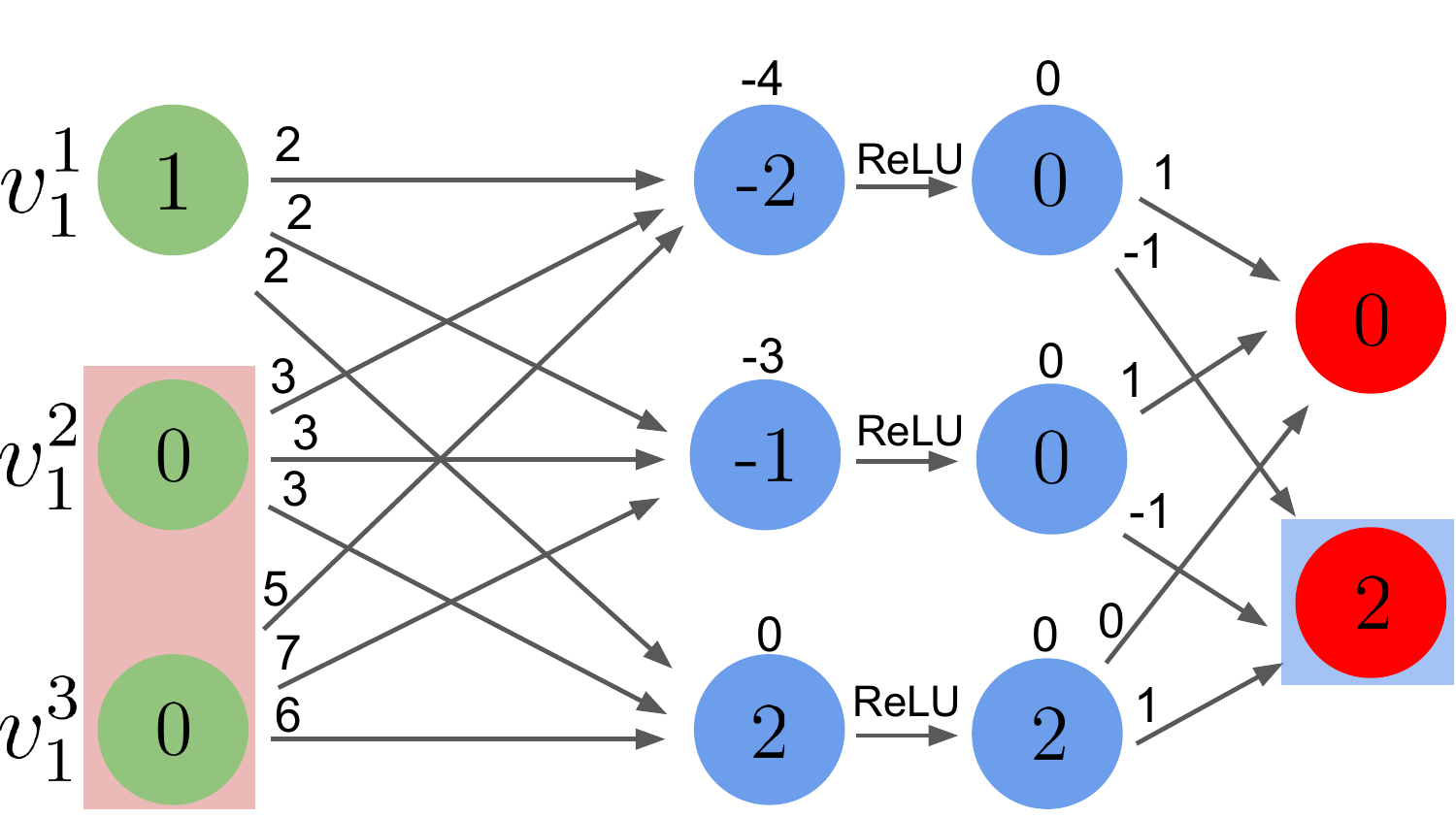}
			\caption{$\{ v_1^2, v_1^3 \}$ is a contrastive example for $V_1=[1,1,1]^T$.}
			\label{fig:neural_network_contrastive_example}
		\end{center}
		\vspace{-1.2cm}  
	\end{wrapfigure}
	A contrastive example for our running example is shown in
	Fig.~\ref{fig:neural_network_contrastive_example}. Notice that
	the question of whether a set is a contrastive example can be
	encoded into a verification query
	$\langle P,N,Q\rangle = \langle(F\setminus C)=v,N,Q_{\neg
		c}\rangle$,
	where a \sat{} result indicates that $C$ is a contrastive
	example.
	As with explanations, smaller contrastive examples are
	more valuable than large ones. One useful
	notion is that of a \emph{contrastive singleton}: a contrastive example of
	size one. A contrastive singleton could represent a specific pixel in an image,
	the alteration of which could result in misclassification. Such
	singletons are leveraged in ``one-pixel attacks''~\cite{su2019one} (see
	Fig.~\ref{fig:one-pixel-attack} in the appendix).
	Contrastive singletons have the following important property:
	\begin{lemma}
		\label{singelton_in_all_explanations}
		Every contrastive singleton is contained in all explanations.
	\end{lemma}
	The proof appears in Sec.~\ref{sec:appendix_explanations_contrastive_proofs} of the appendix.  Lemma~\ref{singelton_in_all_explanations} implies that each contrastive
	singleton is contained in all minimal/minimum
	explanations.
	
	We consider also the notion of a \emph{contrastive pair}, which is a 
	contrastive example of size 2. Clearly, for any pair of features $(u,v)$ where
	$u$ or $v$ are contrastive singletons, $(u,v)$ is a contrastive
	pair; however, when we next refer to contrastive pairs, we
	consider only pairs that \emph{do not} contain any contrastive
	singletons. Likewise, for every $k>2$, we can consider contrastive
	examples of size $k$, and we exclude from these any contrastive examples of sizes $1,\ldots,k-1$ as subsets.
	
	We state the following theorem, whose proof also
	appears in Sec.~\ref{sec:appendix_explanations_contrastive_proofs} of the appendix:
	\begin{lemma}
		\label{one_element_of_pair}
  	All explanations contain at least one element of every contrastive pair.
	\end{lemma}
	The theorem can be generalized to any $k>2$; and can be used in
	showing that the \emph{minimum hitting set (MHS)} of all
	contrastive examples is exactly the minimum
	explanation~\cite{ignatiev2016propositional, reiter1987theory}
        (see
	Sec.~\ref{mhs-defintion-appendix} of the appendix). Further, the
	theorem implies a duality between contrastive examples and
	explanations~\cite{ignatiev2020contrastive, ignatiev2015smallest}: a
	minimal hitting set of all contrastive examples constitutes a minimal
	explanation, and a minimal hitting set of all explanations
	constitutes a minimal contrastive example.
	\section{Provable Approximations for Minimal Explanations}
	\label{sec:approximations}
	State-of-the-art approaches for finding minimum explanations
	exploit the MHS duality between explanations and contrastive 
	examples~\cite{ignatiev2019abduction}. The idea is to
	iteratively compute contrastive examples, and then use their
	MHS as an under-approximation for the minimum
	explanation. Finding this MHS is an NP-complete problem, and
	is difficult in practice as the number of contrastive examples
	increases~\cite{gainer2017minimal}; and although the MHS can
	be approximated using maximum satisfiability (MaxSAT) or mixed
	integer linear programming (MILP) solvers~\cite{li2021maxsat,
		ilog2018cplex}, existing approaches tackle simpler ML
	models, such as decision trees~\cite{izza2020explaining,
		ignatiev2018sat}, but face scalability limitations when
	applied to DNNs~\cite{ignatiev2019abduction,
		shi2020tractable}.  Further, enumerating all contrastive
	examples may in itself take exponential time. Finally, recall
	that DNN verification is an NP-Complete
	problem~\cite{katz2017reluplex}; and so dispatching a
	verification query to identify each explanation or
	contrastive example is also very slow, when the feature space
	is large. Finding \emph{minimal} explanations may be
	easier~\cite{ignatiev2019abduction}, but may converge to
	larger and less meaningful explanations, while still requiring
	a linear number of calls to the underlying verifier.  Our
	approach, described next, seeks to mitigate these
	difficulties.
	
	Our overall approach is described in
	Algorithm~\ref{alg:approximation}. It is comprised of two
	separate threads, intended to be run in parallel. The
	\emph{upper bounding thread} (\tub) is responsible for
	computing a minimal explanation. It starts with the
	entire feature space, and then gradually reduces it, until
	converging to a minimal explanation. The size of the presently
	smallest explanation is regarded as an upper bound (\upperb)
	for the size of the minimum explanation. Symmetrically, the
	\emph{lower bounding thread} (\tlb) attempts to construct
	small contrastive sets, used for computing a lower bound
	(\lowerb) on the size of the minimum explanation. Together,
	these two bounds allow us to compute the approximation ratio
	between the minimal explanation that we have discovered and
	the minimum explanation. For instance, given a minimal
	explanation of size 7 and a lower bound of size 5, we can
	deduce that our explanation is at most
	$\frac{\upperb}{\lowerb}=\frac{7}{5}$ times larger than the
	minimum. The two threads
	share global variables that indicate the set of
	contrastive singletons (\singletons), the set of contrastive
	pairs (\pairs), the upper and lower bounds (\upperb, \lowerb),
	and the set of features that were determined not to
	participate in the explanation 
	and are ``free'' to be set to any value (\free). The
	output of our algorithm is a minimal explanation
	(F$\setminus$Free), and the approximation ratio
	($\frac{\upperb}{\lowerb}$). We next discuss each of the two threads in
	detail.
	
	\begin{algorithm}
		\caption{Minimal Explanation Search}\label{alg:approximation}
		\textbf{Input} N (Neural network), F (features), v (input values), c (class prediction)
		\begin{algorithmic}[1]
			\State \singletons, \pairs, \free $\gets \emptyset$, \upperb $\gets |F|$, \lowerb$\gets0$ \Comment{Global variables}
			\State Launch thread \tub
			\State Launch thread \tlb
			\State \Return F$\setminus$Free, $\frac{\upperb}{\lowerb}$
			
		\end{algorithmic}
	\end{algorithm}
	
	\mysubsection{The Upper Bounding Thread ($T_\upperb$).}  This thread, whose pseudocode
	appears in Algorithm~\ref{alg:upper-thread}, follows the 
	framework proposed by Ignatiev et
	al.~\cite{ignatiev2019abduction}: it seeks a
	minimal explanation by starting with the entire feature space, and
	then iteratively attempting to remove individual features. If removing
	a feature allows misclassification, we keep it as
	part of the explanation; otherwise, we
	remove it and continue. This process
	issues a single verification query for each feature, until converging
	to a minimal explanation (lines \ref{lst:line:startregularupper}--\ref{lst:line:endregularupper}). Although this na\"ive search is
	guaranteed to converge to a minimal explanation, it needs not to converge
	to a \emph{minimum} explanation; and so we apply a more sophisticated
	ordering scheme, similar to the one proposed
	by~\cite{ignatiev2019validating}, where we use some heuristic
	model as a way for assigning weights of importance to each
	input feature. We then check the ``least
	important'' input features first, since freeing them has a lower
	chance of causing a misclassification, and they are consequently
	more likely to be successfully removed. We then continue
	iterating over features in ascending order of importance, hopefully
	producing small explanations.
	
	\begin{algorithm}
		\algnewcommand\algorithmicforeach{\textbf{for each}}
		\algdef{S}[FOR]{ForEach}[1]{\algorithmicforeach\ #1\ \algorithmicdo}
		\caption{\tub: Upper Bounding Thread}\label{alg:upper-thread}
		\begin{algorithmic}[1]
			\State{Use a heuristic model to sort $F$'s features by ascending relevance}
			\ForEach {$f \in F$}\label{lst:line:startregularupper}
			\State \explanation$\gets\ $F$\setminus$Free
			\If{\verifyexplanation is \unsat{}}
			\State{\free$\gets\ $\freewithf}
			\State{$\upperb \gets \upperb-1$}
			\EndIf
			\EndFor\label{lst:line:endregularupper}
		\end{algorithmic}
	\end{algorithm}
	\mysubsection{The Lower Bounding Thread (\tlb).}
	\label{lower-code-section}
	The pseudocode for the lower bounding thread (\tlb) appears in
	Algorithm~\ref{alg:lower-thread}. In lines
	\ref{lst:line:lowersingletonstart}--\ref{lst:line:lowersingletonend},
	the thread searches for contrastive singletons.  Neural
	networks were shown to be very sensitive to adversarial
	attacks~\cite{huang2017adversarial} --- slight input
	perturbations that cause misclassification (e.g., the
	aforementioned one-pixel attack~\cite{su2019one}) --- and this
	suggests that contrastive sets, and in particular contrastive
	singletons, exist in many cases.  We observe that identifying
	contrastive singletons is computationally cheap: by encoding
	Eq.~\ref{eq:contrastive_examples} as a verification query,
	once for each feature, we can discover all singletons; and in
	these queries all features but one are fixed, which
	empirically allows verifiers to dispatch them quickly.
	
	\begin{algorithm}
		\algnewcommand\algorithmicforeach{\textbf{for each}}
		\algdef{S}[FOR]{ForEach}[1]{\algorithmicforeach\ #1\ \algorithmicdo}
		\caption{$T_{LB}$: Lower Bounding Thread}\label{alg:lower-thread}
		\begin{algorithmic}[1]
			
			\ForEach {$f \in F$} \Comment{Find all singletons}\label{lst:line:lowersingletonstart}
			\If{\verifysingleton is \sat{}}
			\State{\singletons$\gets\ $\singletons $\cup\ \{f\}$}
			\State{\lowerb $\gets$ \lowerb $+1$}
			\EndIf
			\EndFor\label{lst:line:lowersingletonend}
			\\
			\State{\allpairs $\gets$ \allpairsinput}
			\ForEach{\singlepair $\in$ \allpairs} \Comment{Find all pairs}\label{lst:line:lowerpairsstart}
			\If{\verifypair is \sat{}}
			\State{\pairs $\leftarrow$\pairs $\cup\ \{(a,b)\}$}
			\EndIf
			\EndFor\label{lst:line:lowerpairsend}
			\State{\lowerb $\gets$ \lowerb $+$ \mvc}
			
		\end{algorithmic}
	\end{algorithm}
 
	The rest of \tlb{}  (lines \ref{lst:line:lowerpairsstart}--\ref{lst:line:lowerpairsend}) performs
	a similar process, but with contrastive pairs (which do not contain
	contrastive singletons as one of their features). We use verification
	queries to identify all such pairs, and then attempt to find their
	MHS.  We observe that finding the MHS of all contrastive pairs is
	the 2-MHS problem, which is a reformalization of the
	\emph{minimum vertex cover} problem (see Sec.~\ref{mhs-defintion-appendix} of the appendix).  Since this is an easier problem than the general MHS
	problem, solving it with MAX-SAT or MILP often converges 
	quickly. In addition, the minimum vertex cover algorithm has a linear
	2-approximating greedy algorithm, which can be used for finding a
	lower bound in cases of large feature spaces.
	
	More formally, \tlb performs an efficient
	computation of the following bound:
	\begin{equation}
		\begin{aligned}
			\lowerb = |\singletons|+|\mvc|\leq\mhscxp=\minimumexplanation
		\end{aligned}
	\end{equation}
	where MVC is the minimum vertex cover, Cxps denotes the set of all
	contrastive examples, and \minimumexplanation is the size of the minimum
	explanation.
	
	It is worth mentioning that this approach can be extended to use
	contrastive examples of larger sizes ($k=3,4,\ldots$), as
        specified in Sec.~\ref{extending-tlb-appendix} of the appendix. The fact that small contrastive examples, such as singletons, exist in large, state-of-the-art DNNs with large
    inputs~\cite{su2019one, garg2020bae} suggests
    that useful approximations exist in large DNNs. In our
	experiments, we observed that using only singletons and pairs
	affords good approximations, without incurring overly expensive
	computations by the underlying verifier.

	\section{Finding Minimal Explanations Efficiently}
	\label{heuristics_local_search}
	
	Algorithm~\ref{alg:approximation} is the backbone of our
	approach, but it suffers from limited scalability ---
	particularly, in \tub{}. As the
	execution of \tub{} progresses, and as additional features are
	``freed'', the quickly growing search space slows down the
	underlying verifier. Here we propose three different methods
	for expediting this process, by reducing the number of
	verification queries required.
	
	\mysubsection{Method 1: Using Information from \tlb{}.}  We
	suggest to leverage the contrastive examples found by \tlb{}
	to expedite \tub{}. The process is described in
	Algorithm~\ref{alg:upper-thread-lower-information}. In line
	\ref{lst:line:lowerinfosingletons}, \tlb{} is queried for
	the current set of contrastive singletons, which we know must
	be part of any minimal explanation. These are subtracted
	from the \featuresleft set (features left for \tub{} to
	query), and consequently will not be added to the
	\free set --- i.e., they are 
	marked as part of the current explanation. In addition, for
	any contrastive pair $(a,b)$ found by \tlb{},
	either $a$ or $b$ must appear in any minimal explanation;
	and so, our algorithm skips checking the case where both $a$
	and $b$ are removed from F (Line~\ref{lst:line:lowerinfopairs}). (the method could also be
	extended to contrastive sets of greater cardinality.)
	
	\begin{algorithm}
		\algnewcommand\algorithmicforeach{\textbf{for each}}
		\algdef{S}[FOR]{ForEach}[1]{\algorithmicforeach\ #1\ \algorithmicdo}
		\caption{\tub using information from \tlb}\label{alg:upper-thread-lower-information}
		\begin{algorithmic}[1]
			\State{Use a heuristic model to sort $F$ by ascending relevance}
			\State \featuresleft $\gets\ $F$\setminus$\singletons
			\ForEach{f $\in$ \featuresleft}\label{lst:line:lowerinfosingletons}
			\State \explanation$\gets\ $F$\setminus$Free
			\If{\verifyexplanation is \unsat{}}
			\State{\free$\gets\ $\freewithf}
			\State{$\upperb \gets \upperb-1$}
			\State{Delete all features in a pair with f from \featuresleft}\label{lst:line:lowerinfopairs}
			\EndIf
			\EndFor
		\end{algorithmic}
	\end{algorithm}
	
	\mysubsection{Method 2: Binary Search.}
	Sorting the features being considered
	in ascending order of importance can have a
	significant effect on the size of the explanation found by Algorithm~\ref{alg:upper-thread}.  Intuitively, a ``perfect''
	heuristic model would assign the greatest weights to  all
	features in the minimum explanation, and so traversing features in
	ascending order 
	would first discover all the features that can be removed (\unsat{}
	verification queries), followed by all
	the features that belong in the explanation (\sat{} queries).
	In this case, a sequential traversal of the features in ascending
	order is quite wasteful, and it is much better to perform a binary
	search to find the point where the answer flips from \unsat{}
	to \sat{}. 
	
	
	Of course, in practice, the heuristic models are not perfect,
	leading to potential cases with multiple ``flips'' from \sat{}
	to \unsat{}, and vice versa. Still, if the heuristic is good
	in practice (which is often the case; see
	Sec.~\ref{sec:Evaluation}), these flips are scarce. Thus, we
	propose to perform multiple binary searches, each time
	identifying one \sat{} query (i.e., a feature added to the
	explanation). Observe that each time we hit an \unsat{} query,
	this indicates that all the queries for features with
	lower priorities would also yield \unsat{} --- because if
	``freeing'' multiple features cannot change the
	classification, changing fewer features certainly
	cannot. Thus, we are guaranteed to find the first \sat{} query
	in each iteration, and soundness is maintained. This process is
	described in Algorithm.~\ref{alg:upper-thread-binary-search} and in Fig.\ref{fig:binary-search} in the appendix.

	\mysubsection{Method 3: Local-Singleton Search.}  Let $N$ be a
	DNN, and let $x$ be an input point whose classification we
	seek to explain. As part of Algorithm~\ref{alg:upper-thread},
	\tub{} iteratively ``frees'' certain input features,
	allowing them to take arbitrary values, as it continues to
	search for features that must be included in the
	explanation. The increasing number of free features enlarges
	the search space that the underlying verifier must traverse,
	thus slowing down verification.  We propose to leverage the
	hypothesis that input points nearby $x$ that are misclassified
	tend to be clustered; and so, it is beneficial to
	\emph{fix} the free features to ``bad'' values, as opposed to
	letting them take on arbitrary values. We speculate that this
	will allow the verifier to discover satisfying
	assignments much more quickly.
	
	This enhancement is shown in
	Algorithm~\ref{alg:upper-thread-local-singelton-search}. Given
	a set \free of features that were previously freed, we fix
	their values according to some satisfying assignment
	previously discovered. Thus, the verification of any new
	feature that we consider is similar to the case of searching
	for contrastive singletons, which, as we already know, is
	fairly fast. See Fig.~\ref{fig:local-singelton-search} in the appendix for an illustration. The process can be improved further by
	fixing the freed features to small neighborhoods of the
	previously discovered satisfying assignment (instead of its
	exact values), to allow some flexibility while still keeping
	the query's search space small.
	
	\begin{algorithm}[!htp]
		
		\algnewcommand\algorithmicforeach{\textbf{for each}}
		\algdef{S}[FOR]{ForEach}[1]{\algorithmicforeach\ #1\ \algorithmicdo}
		\caption{\tub using local-singleton search}\label{alg:upper-thread-local-singelton-search}
		\begin{algorithmic}[1]
			\State{Use a heuristic model to sort $F$ by ascending relevance}
			\State \featuresleft $\gets\ $F$\setminus$\singletons
			\ForEach {f $\in$ \featuresleft}
			\State \explanation$\gets\ $F$\setminus$Free
			\If{\verifyexplanation is \unsat{}}
			\State{\free$\gets\ $\freewithf}
			\State{$\upperb \gets \upperb-1$}
			\Else
			\State {Extract counter example C}
			\State \localsingletons $\gets \emptyset$
			\ForEach {$f^\prime$ $\in$ \featuresleft}
			\If{\verifylocalsingletons is \sat{}}
			\State {\localsingletons $\gets$ \localsingletons$\cup\ \{f^\prime\}$}
			\EndIf
			\EndFor
			\State \featuresleft $\gets$
			\featuresleft$\setminus\ $\localsingletons
			\EndIf
			\EndFor
		\end{algorithmic}
	\end{algorithm}

	\section{Minimal Bundle Explanations}
	
	\label{sec:bundles}
	

\begin{wrapfigure}[11]{r}{0.2\textwidth}
	\vspace{-2.2cm}
	\begin{center}
		\includegraphics[width=0.2\textwidth]{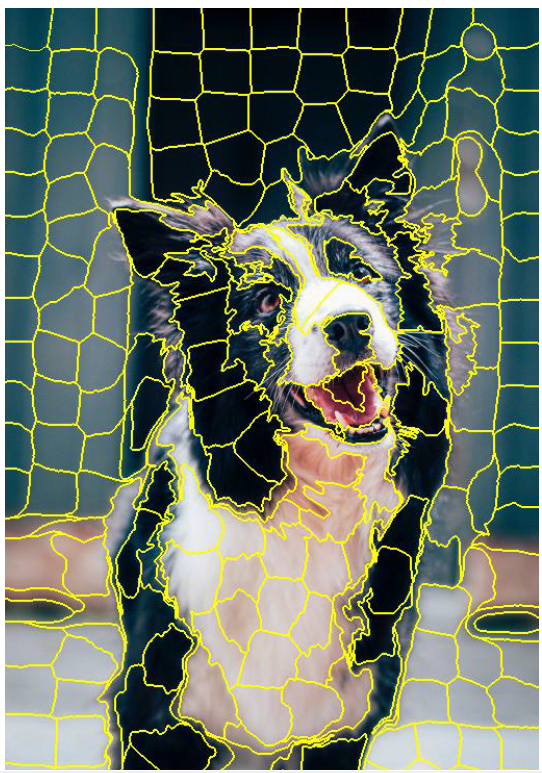}
		\caption{Partition input's features into bundles.}
		\label{fig:bundle2}
	\end{center}
\end{wrapfigure}

So far, we presented methods for generating explanations within a
given approximation ratio of the minimum explanation
(Sec.~\ref{sec:approximations}), and for expediting the computation of
these explanations (Sec.~\ref{heuristics_local_search}) --- in order
to improve the scalability of our explanation generation
mechanism. Next, we seek to tackle the second challenge from
Sec.~\ref{sec:Introduction}, namely that these explanations may be too
low-level for many users. To address this challenge, we focus on
\emph{bundles}, which is a topic well covered in the
ML~\cite{stutz2018superpixels} and heuristic XAI
literature~\cite{ribeiro2016should, lundberg2017unified} (commonly known as ``super-pixels'' for computer-vision tasks).
Intuitively, bundles are a partitioning of the features into disjoint
sets (an illustration appears in Fig.~\ref{fig:bundle2}). The idea,
which we later validate empirically, is that providing explanations in
terms of bundles is often easier for humans to comprehend. As an added
bonus, using bundles also curtails the search space that the
verifier must traverse, expediting the process even
further.


Given a feature space $F=\{1,...,m\}$, a bundle $b$ is just a
subset $b\subseteq F$. When dealing with the set of all bundles
$B=\{b_{1},b_{2},...b_{n}\}$, we require that they form a partitioning
of $F$, namely $F= \cupdot b_i$.  We define a \emph{bundle
	explanation} $E_B$ for a classification instance $(v,c)$ as a subset
of bundles, $E_B\subseteq B$, such that:
\begin{equation}
	\label{eq:bundle_explanation}
	\forall(x\in \mathbb{F}).[\wedge_{i\in \cup E_B}(x_{i}=v_{i})\to(N(x)=c)]
\end{equation}
The following theorem then connects bundle explanations
and explicit, non-bundle explanations:
\begin{theorem}
	\label{bundle_no_bundle_explanation}
	The union of features in a bundle explanation is an explanation. 
\end{theorem}
The proof directly follows from Eqs.~\ref{eq:explanation}
and~\ref{eq:bundle_explanation}.
We note that this definition of bundles implies that features that are
not part of the bundle explanation (i.e. features contained in
\emph{``free'' bundles}) are ``free'' to be set to any possible
value. Another possible alternative for defining bundles could be to
allow features in ``free'' bundles to only change in the same,
coordinated manner. We focus here on the former definition, and leave
the alternative definition for future work. 

Many of the aforementioned results and definitions for
explanations can be extended to bundle explanations. In
a similar manner to Eq.~\ref{eq:bundle_explanation}, we can
define the notions of minimal and minimum bundle
explanations, a contrastive bundle singleton, and contrastive
bundle pairs (see
Sec.~\ref{minimal-bundle-explanation-appendix} of the appendix). Theorems~\ref{singelton_in_all_explanations}
and~\ref{one_element_of_pair} can be extended to bundle
explanations in a straightforward manner. It then follows that
all bundle explanations contain all contrastive singleton bundles, and that 
all bundle explanations contain at least one bundle of any contrastive bundle pair.

Our method from Secs.~\ref{sec:approximations} and~\ref{heuristics_local_search} can be similarly performed on bundles rather than on features, and \tub would then be used for calculating a minimal bundle explanation, rather than a minimal explanation. Regarding the aforementioned approximation ratio,
we discuss and evaluate two different methods for obtaining it. The first, natural approach is to apply
our techniques from Sec.~\ref{sec:approximations} on bundle
explanations, thus obtaining a provable approximation for a
\emph{minimum bundle explanation}. The upper bound is trivially derived by the size of the bundle explanation found by \tub, whereas the lower bound calculation requires assigning a cost to each bundle,
representing the number of features it contains. This is
done via a known notion of \emph{minimum hitting sets of
	bundles (MHSB)~\cite{angel2009minimum}} and using minimum
\emph{weighted} vertex cover for the approximation of
contrastive bundle pairs. This method, which is almost
identical to the one mentioned in
Sec.~\ref{sec:approximations}, is formalized in
Sec.~\ref{sec:bundle_approximation_appendix} of the appendix.

%
%
The second approach is to calculate an approximation ratio with respect to a regular, non-bundle
minimum explanation. The minimal bundle explanation found by \tub is
an upper bound on the minimum non-bundle explanation following theorem
\ref{eq:bundle_explanation}. For computing a lower bound,
we can analyze contrastive bundle
examples; extract from them contrastive non-bundle
examples; and then use the duality property, compute an MHS of these
contrastive examples, and derive lower bounds for the size of the
minimum explanation. We formalize techniques for performing this
calculation in Sec.~\ref{sec:bundle_approximation_appendix} of the appendix.

\section{Evaluation}
\label{sec:Evaluation}

\mysubsection{Implementation and Setup.}
For evaluation purposes, we created a proof-of-concept implementation
of our approach as a Python framework.
Currently, the framework uses the Marabou verification
engine~\cite{katz2019marabou} as a backend, although other engines may
be used. Marabou is a Simplex-based DNN verification framework that is
sound and complete~\cite{katz2019marabou, WuOzZeIrJuGoFoKaPaBa20,
	katz2017reluplex, WuZeKaBa22, AmWuBaKa21, KaBaDiJuKo21,amir2022verifire}, and which
includes support for proof production~\cite{IsBaZhKa22}, abstraction~\cite{WaPeWhYaJa18, SiGePuVe19, ElGoKa20, OsBaKa22,
	ZeWuBaKa22, ElCoKa22}, and optimization~\cite{StWuZeJuKaBaKo21}; and has
been used in various settings, such as ensemble
selection~\cite{AmKaSc22}, simplification~\cite{LaKa21,
	GoFeMaBaKa20} repair~\cite{ReKa22, GoAdKeKa20}, and verification of
reinforcement-learning based systems~\cite{AmCoYeMaHaFaKa22, AmScKa21,
	ElKaKaSc21}.
For sorting features by their relevance, we used the popular
XAI method LIME~\cite{ribeiro2016should}; although again, other
heuristics could be used. The MVC was calculated using the classic
2-approximating greedy algorithm.  All experiments reported were
conducted on x86-64 Gnu/Linux-based machines, using a single Intel(R)
Xeon(R) Gold 6130 CPU @ 2.10GHz core, with a 1-hour timeout.

\mysubsection{Benchmarks.}
As benchmarks, we used DNNs trained over the MNIST dataset for
handwritten digit recognition~\cite{lecun1998mnist}. These networks
classify $28\times 28$ grayscale images into the digits
$0,\ldots,9$. Additionally, we used DNNs trained over the
Fashion-MNIST dataset~\cite{xiao2017fashion}, which classify
$28\times 28$ grayscale images into 10 clothing categories
(``Dress'', ``Coat'', etc.)
For each of these datasets we
trained a DNN with the following architecture:
\begin{inparaenum}[(i)]
	\item an input layer (which corresponds to the image) of size 784;
	\item a fully connected hidden layer with 30 neurons;
	\item another fully connected hidden layer, with 10 neurons; and
	\item a final, softmax layer with 10 neurons, corresponding to the 10
	possible output classes.
\end{inparaenum}
The accuracy of the MNIST DNN was 96.6\%, whereas that of the
Fashion-MNIST DNN was 87.6\%. (We note that we configured LIME to
ignore the external border pixels of each input, as
these are not part of the actual image.)

In selecting the classification instances to be explained for
these networks, we targeted input points where the network was
not confident --- i.e., where the winning label did not win by
a large margin. The motivation for this choice is that
explanations are most useful and relevant in cases where the
network's decision is unclear, which is reflected in
lower confidence scores. Additionally, explanations of instances with
lower confidence tend to be larger, facilitating the process of extensive experimentation.
We thus selected the 100 inputs from the MNIST and the
Fashion-MNIST datasets where the networks demonstrated the
lowest confidence scores --- i.e., where the difference between the winning output score and the runner-up class score was minimal. 

\mysubsection{Experiments.}  Our first goal was to compare our
approach to that of Ignatiev et
al.~\cite{ignatiev2019abduction}, which is the current state
of the art in verification-based explainability of
DNNs. Other approaches consider other ML types, such as descision trees~\cite{izza2020explaining, ignatiev2018sat}, or focus on alternative definitions for abductive explanations~\cite{wu2022verix, la2021guaranteed} and are thus not comparable. Because the implementation used 
in~\cite{ignatiev2019abduction} is unavailable, we implemented
their approach, using Marabou as the underlying verifier for a
fair comparison. In addition, we used the same heuristic
model, LIME, for sorting the input features' relevance.
Fig.~\ref{fig:final_basic_results} depicts a comparison of the
two approaches, over the MNIST benchmarks. The Fashion-MNIST
results were similar, but since the Fashion-MNIST network had lower
accuracy it tended to produce larger explanations with lower
run-times, resulting in less meaningful evaluations (due to space
limitations, these results appear in Fig.~\ref{fig:final_basic_results_appendix_fashion} in the
appendix).  We compared the
approaches according to two criteria: the portion of input features
whose participation in the explanation was verified, over time (part (a) of
Fig.~\ref{fig:final_basic_results}), and the average size of the presently obtained
explanation over time, also presented as a fraction of the
total number of input features
(part (b)). The results indicate that
our method significantly improves over the state of the art, verifying
the participation of 40.4\% additional features, on average, and producing explanations that are 9.7\% smaller, on average, at the end of the 1-hour time limit. Furthermore,
our method timed out on 10\% fewer benchmarks. We regard this
as compelling evidence of the potential of our approach to produce more efficient verification-based XAI.

\begin{figure}[!tbp]
	\centering
	\subfloat[Average portion of features verified to participate
	in the explanation.]{\includegraphics[width=0.43\textwidth]{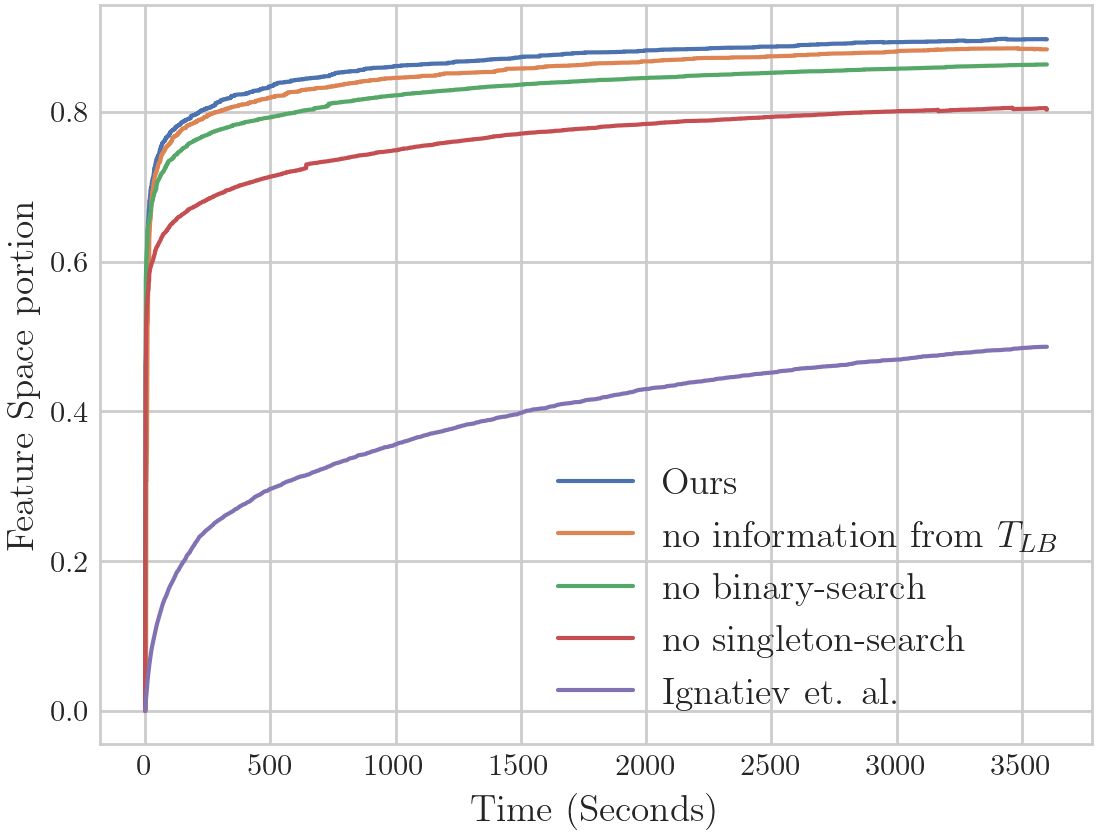}\label{fig:numofqueries}}
	\hfill
	\subfloat[Average explanation size.]{\includegraphics[width=0.43\textwidth]{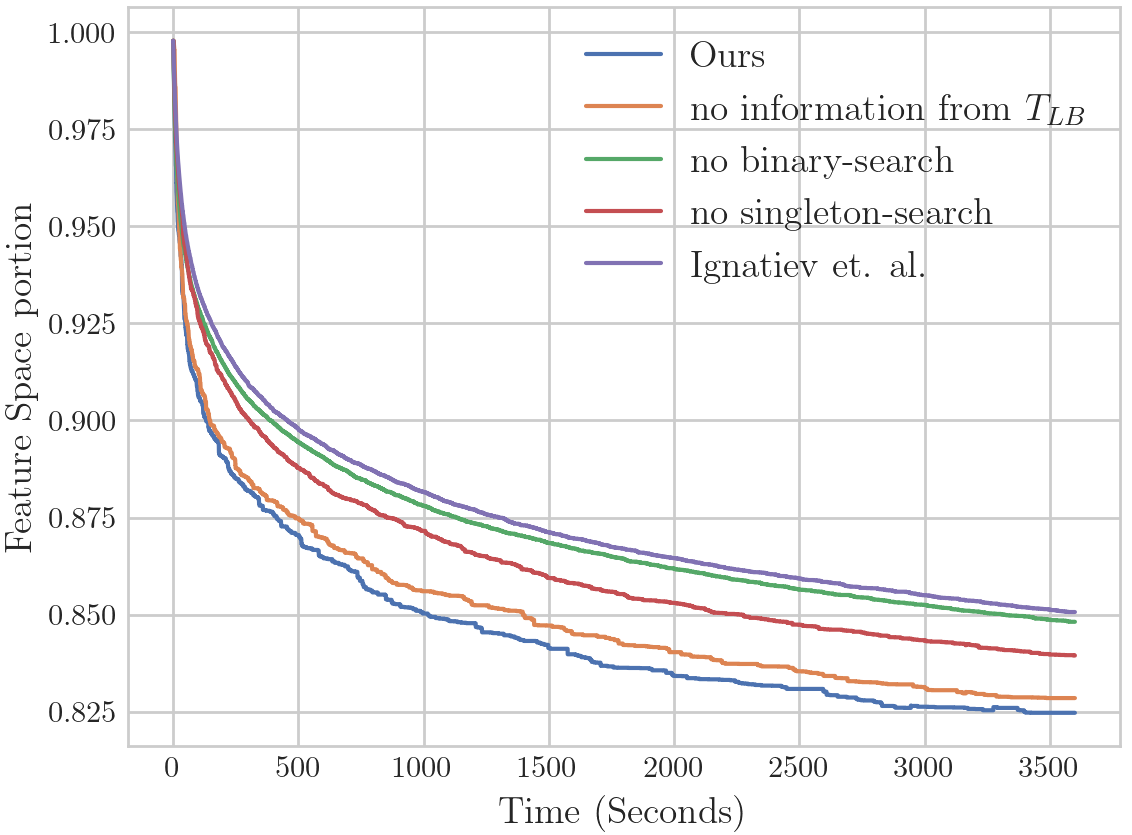}\label{fig:explanation_size}}
	\caption{Our full and ablation-based results, compared to the
          state of the art for finding minimal explanations on the MNIST dataset.}
	\label{fig:final_basic_results}
      \end{figure}

We also looked into comparing our approach to heuristic,
non-verification-based approaches, such as LIME itself; but
these comparisons did not prove to be meaningful, as the
heuristic approaches typically solved benchmarks very quickly,
but very often produced incorrect explanations. This matches the
findings reported in previous
work~\cite{ignatiev2019validating, camburu2019can}.

Next, we set out to evaluate the contribution of each of the
components implemented within our framework to  overall
performance, using an ablation study. Specifically, we ran our
framework with each of the components mentioned in
Sec.~\ref{heuristics_local_search}, i.e.
\begin{inparaenum}[(i)]
	\item information exchange between \tub{} and \tlb{};
	\item the binary search in \tub{}; and
	\item local-singleton search,
\end{inparaenum}
turned off. The results on the MNIST benchmarks appear in
Fig.~\ref{fig:final_basic_results}; see
Fig.~\ref{fig:final_basic_results_appendix_fashion} in the appendix for the Fashion-MNIST results.  Our experiments revealed that each 
of the methods  mentioned in
Sec.~\ref{heuristics_local_search} had a favorable impact on both the
average portion of features verified, and the
average size of the discovered explanation, over
time. Fig \ref{fig:numofqueries} indicates that the local-singleton
search method, used for efficiently proving that
features are bound to be \emph{included} in the explanation, was the
most significant in reducing the number of features remained for
verifying, thus substantially increasing the portion of verified
features. Moreover, Fig.~\ref{fig:explanation_size} indicates that the
binary search method, which is used for grouping \unsat{} queries and
proving the \emph{exclusion} of features from the explanation,
was the most significant for more efficiently obtaining
smaller-sized explanations, over time.

\begin{wrapfigure}{r}{0.4\textwidth}
	\vspace{-1cm}
	\centering
	\includegraphics[width=0.4\textwidth]{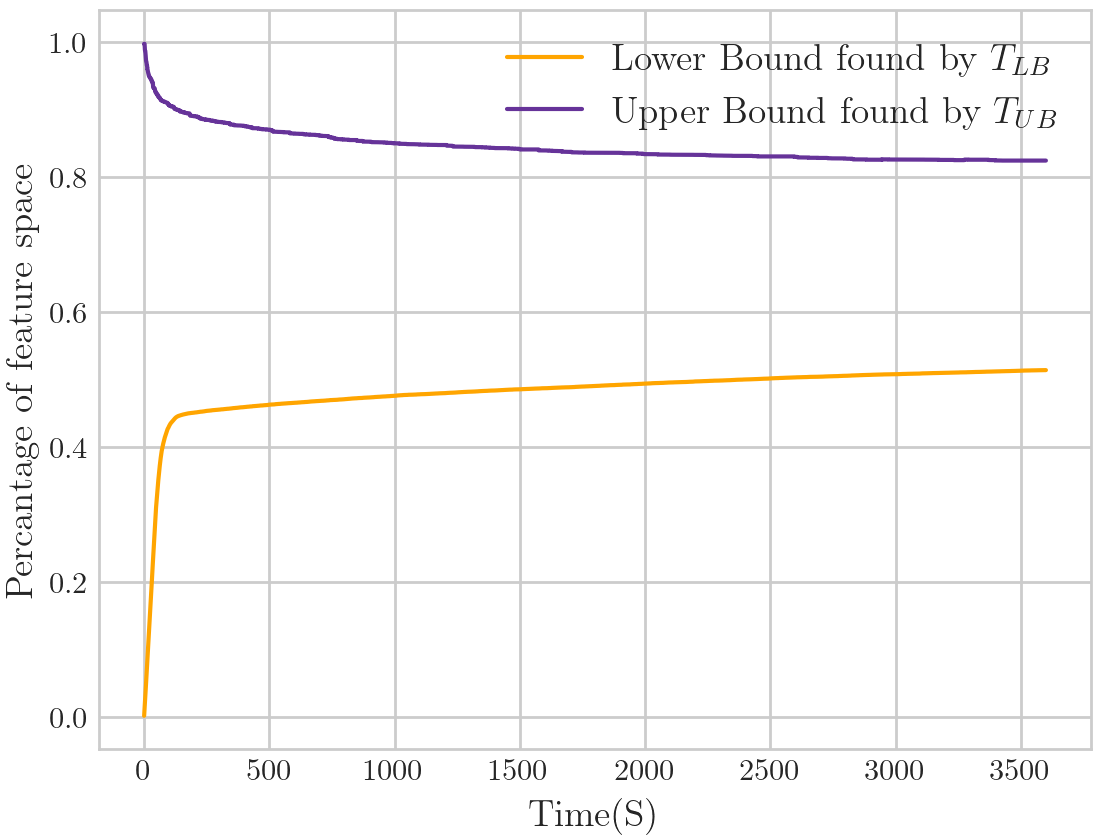}
	\caption{Average approximation of \emph{minimum} explanation over time.}
	\label{fig:minimum_approximation}
 	\vspace{-0.6cm}

\end{wrapfigure}
Our second goal was to evaluate the quality of the \emph{minimum} explanation approximation of our method (using the lower/upper bounds) over time. Results are averaged over all benchmarks of the MNIST dataset
and are presented in Fig.~\ref{fig:minimum_approximation}
(similar results on Fashion-MNIST appear in
Fig.~\ref{fig:fashion_upper_lower} in the appendix). The upper bound
represents the average size of the explanation discovered by
\tub over time, whereas the lower bound represents the average lower bound discovered by
\tlb{} over time. It can be seen that
initially, there
is a steep increase in the size of the lower bound, as
\tlb{} discovered many contrastive singletons. Later, as
we begin iterating over contrastive pairs, the
verification queries take longer to solve,
and progress becomes slower. The average approximation ratio
achieved after an hour was
1.61 for MNIST and 1.19 for Fashion-MNIST.

For our third experiment, we set out to assess the
improvements afforded by bundles. We repeated the
aforementioned experiments, this time using sets of features
representing bundles instead of the features themselves. The
segmentation into bundles was performed using the
\emph{quickshift} method~\cite{vedaldi2008quick}, with LIME
again used for assigning relevance to each
bundle~\cite{ribeiro2016should}.  We approximate the sizes of
the bundle explanations in terms of both the minimum bundle
explanation as well as the minimum (non-bundle) explanation
(as mentioned in Sec.~\ref{sec:bundles} and in Sec.~\ref{sec:bundle_approximation_appendix} of the appendix).
The bundle configuration showed drastic efficiency improvements, with
none of the experiments timing out within the 1-hour time limit, thus improving the portion of timeouts
on the MNIST dataset by 84\%. The efficiency improvement was obtained
at the expense of explanation size, resulting in a decrease of 352\% in the approximation ratios 
obtained for MNIST and 39\% for
Fashion-MNIST. Nevertheless, when calculating the approximation in
terms of the \emph{minimum bundle explanation}, an increase of 12\%
and 8\% was obtained for MNIST and Fashion-MNIST (results are
summarized in Table.~\ref{tab:analysis} in the appendix). For a visual
evaluation, we performed the same set of experiments for both bundle and
non-bundle implementations, using instances with high confidence rates to
obtain smaller-sized explanations that could be more easily
interpreted. A sample of these results is presented in
Fig.~\ref{fig:bundle_evaluation}. Empirically, we observe that the
bundle-produced explanations are less complex and more comprehensible. 
\begin{figure}[!tbp]
	\centering
	\subfloat[Original Image]{\includegraphics[width=0.33\textwidth]{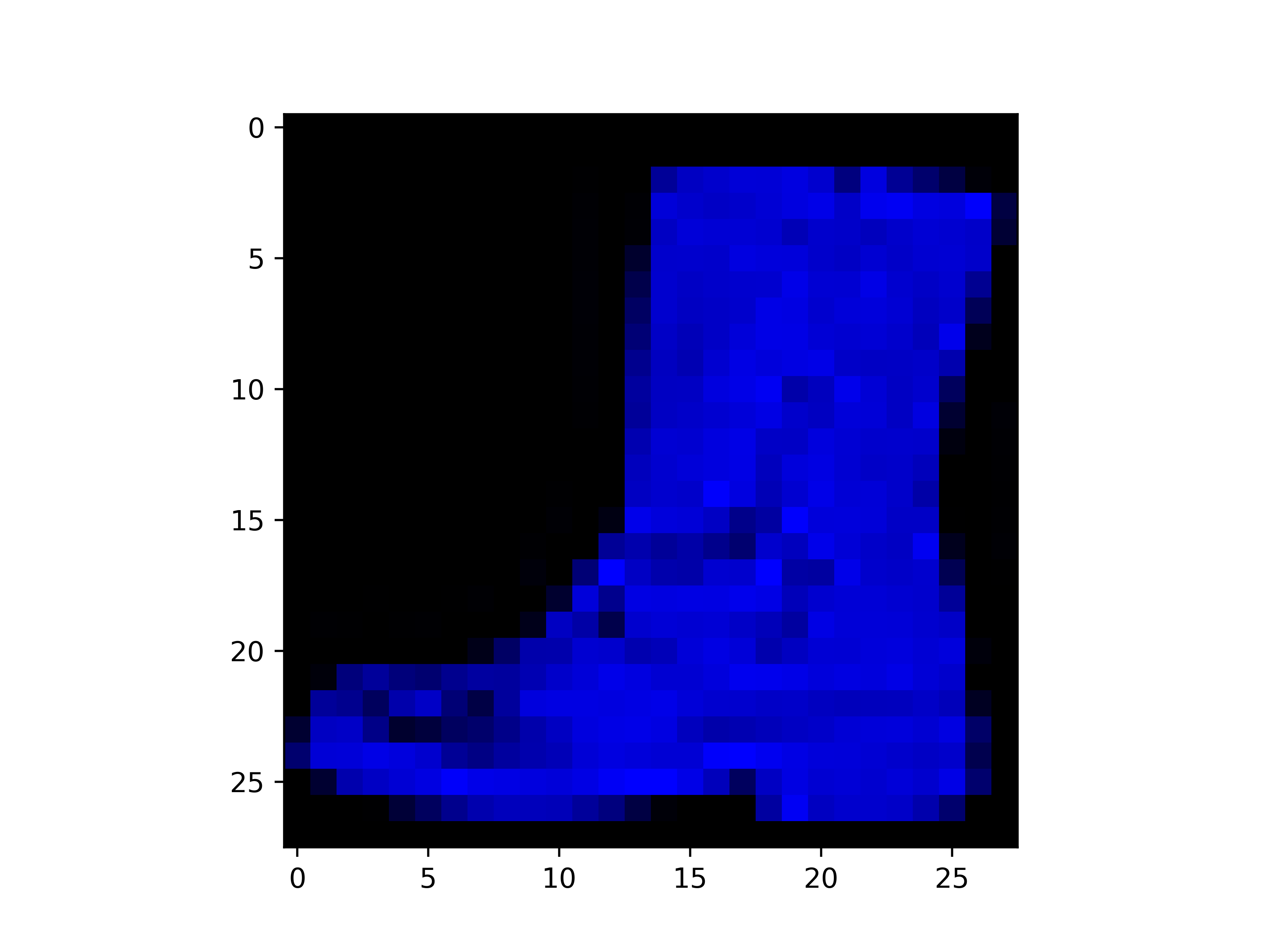}}
	\hfill
	\subfloat[Explanation]{\includegraphics[width=0.33\textwidth]{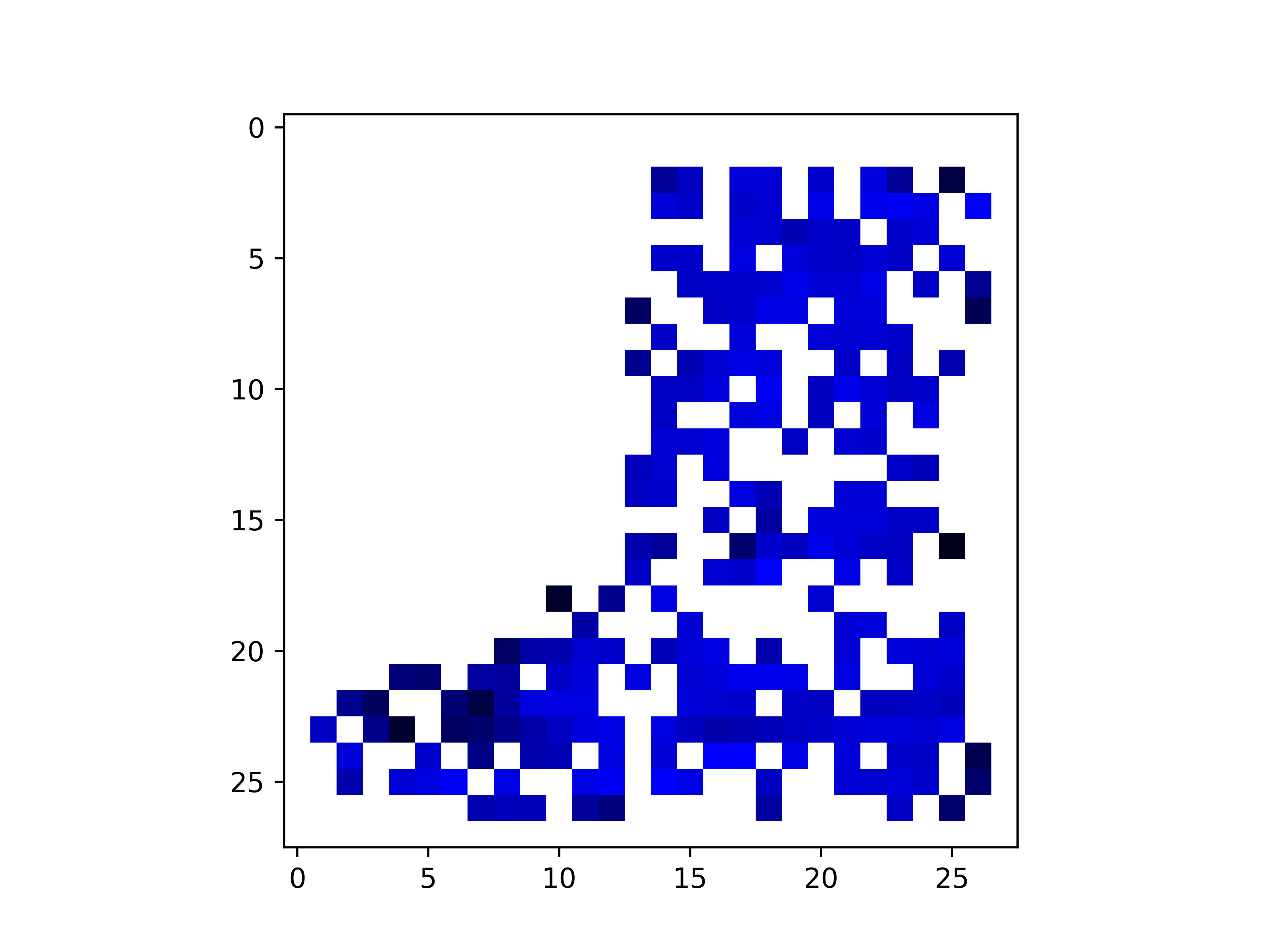}}
	\hfill
	\subfloat[Bundle explanation]{\includegraphics[width=0.33\textwidth]{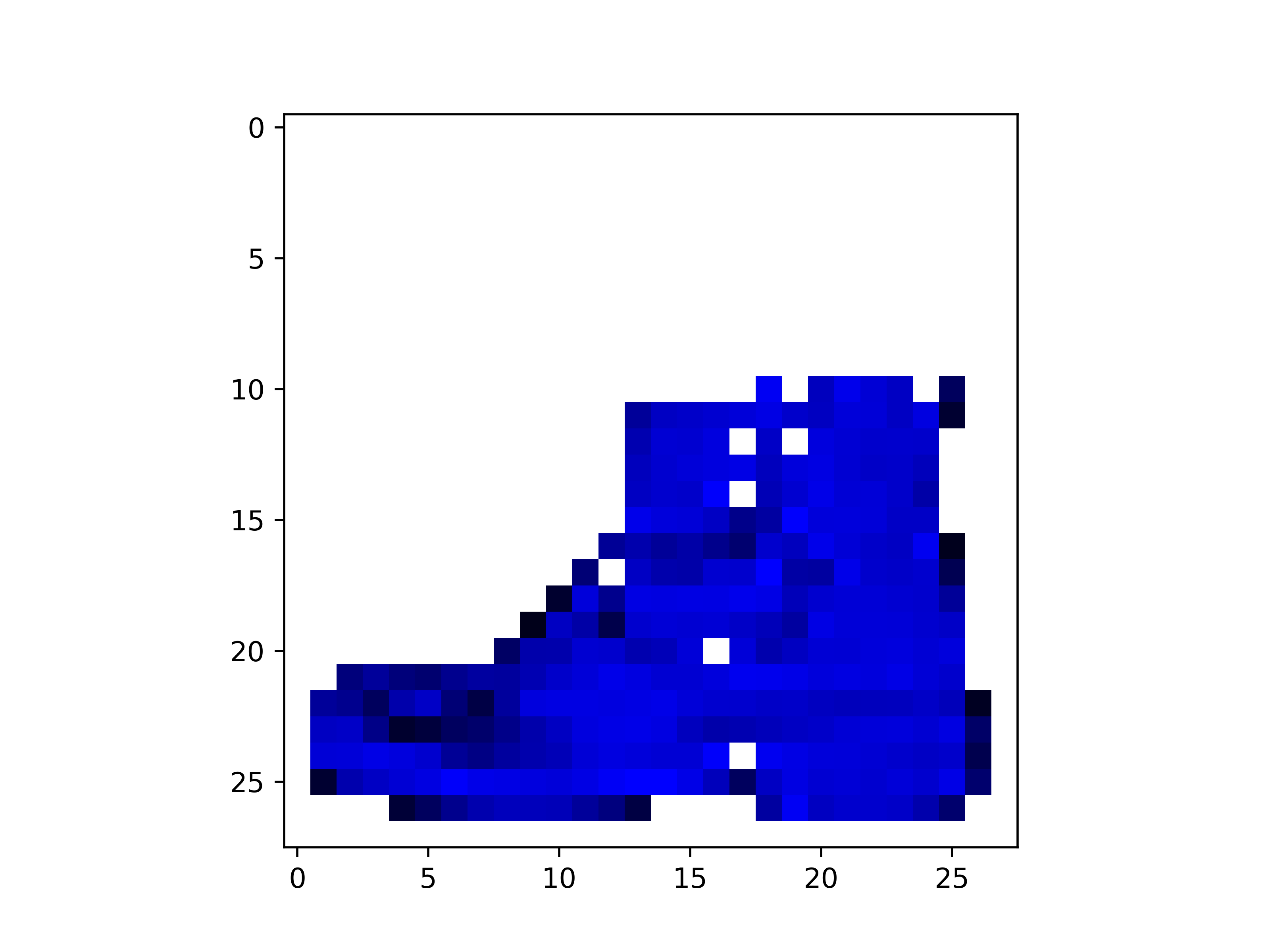}}
	\caption{Minimal explanations and bundle explanations found by our method on the Fashion-MNIST dataset. White pixels are not part of the explanation.}
	\label{fig:bundle_evaluation}
\end{figure} 


Overall, we regard our results as compelling evidence that
verification-based XAI can soundly produce meaningful
explanations, and that our improvements can indeed
significantly improve its runtime.
\section{Related Work}
\label{sec:RelatedWork}
Our work is another step in the ongoing quest for formal
explainability of DNNs, using
verification~\cite{ignatiev2019abduction, shi2020tractable, ignatiev2020towards, fel2022don}. Related approaches have applied
enumeration of contrastive
examples~\cite{ignatiev2020contrastive,
	ignatiev2019abduction}, which is also an ingredient of our
approach. Other approaches focus on producing abductive explanations around an epsilon environment~\cite{wu2022verix, la2021guaranteed}.  Similar work has been carried out for decision
sets~\cite{ignatiev2018sat}, lists~\cite{ignatiev2021sat} and
trees~\cite{izza2020explaining}, where the problem appears to
be simpler to solve~\cite{izza2020explaining}. Our work here
tackles DNNs, which are known to be more difficult to
verify~\cite{katz2017reluplex}.

Prior work has also sought to produce approximate
explanations, e.g., by using $\delta$-relevant
sets~\cite{waeldchen2021computational,
	izza2021efficient}. This line of work has focused on
probabilistic methods for generating explanations, which
jeopardizes soundness.  There has also been extensive work in
heuristic XAI~\cite{ribeiro2016should, ribeiro2018anchors,
	smilkov2017smoothgrad, lundberg2017unified}, but here, too, the produced
explanations are not guaranteed to be correct.



\section{Conclusion}

Although DNNs are becoming crucial components of safety-critical
systems, they remain ``black-boxes'', and cannot be interpreted by
humans.  Our work seeks to mitigate this concern, by providing
formally correct explanations for the choices that a DNN makes. Since
discovering the minimum explanations is difficult, we focus on
approximate explanations, and suggest multiple techniques for
expediting our approach --- thus significantly improving over the
current state of the art. In addition, we propose to use bundles to
efficiently produce more meaningful explanations. Moving forward, we
plan to leverage lightweight DNN verification techniques for improving
the scalability of our approach~\cite{LiArLaBaKo20}, as well as extend
it to support additional DNN architectures.
\label{sec:Conclusion}



{
	\bibliographystyle{abbrv}
	\bibliography{references}
}

\newpage
\appendix
\setcounter{section}{0}
\def\thesection{\Alph{section}}

\noindent
{\huge{Appendix}}

\section{Explanations and Contrastive Examples}
\label{sec:appendix_explanations_contrastive_proofs}

\setcounter{lemma}{0}

\begin{lemma}
	Every contrastive singleton is contained in all explanations.
\end{lemma}
\begin{proof}
	Let $N$ be a classification function, $(v,c)$ a classification
	instance, and $S \in F$ a contrastive singleton. Assume towards
	contradiction that $S$ is not contained in some explanation $E
	\subseteq F$ of instance $(v,c)$, 
	namely $S \not\subseteq E$. This means that $E \subseteq F\setminus S$, and from Theorem~\ref{eq:contrastive_examples} we can conclude that:
	\begin{equation*}
		\exists(x\in \mathbb{F}).[\wedge_{i\in E}(x_{i}=v_{i})\wedge(N(x)\neq c)]
	\end{equation*}
	which is a contradiction to Eq.~\ref{eq:explanation}.\qed
\end{proof}

\begin{lemma}
All explanations contain at least one element of every contrastive pair.
\end{lemma}
\begin{proof}
	The proof is almost identical to that of
	Lemma~\ref{singelton_in_all_explanations}. Let $N$ be a
	classification function, $(v,c)$ a classification instance,
	$P \in F$ a contrastive pair and $E \subseteq F$ an explanation of
	instance $(v,c)$. Assume towards contradiction that there does not
	exist any element from the pair $P$ that is contained in $E$,
	meaning that $P\not\subseteq E$. Hence, $E \subseteq F\setminus P$,
	and from Theorem~\ref{eq:contrastive_examples} we can conclude that:
	\begin{equation*}
		\exists(x\in \mathbb{F}).[\wedge_{i\in E}(x_{i}=v_{i})\wedge(N(x)\neq c)]
	\end{equation*}
	which is a contradiction to Eq.~\ref{eq:explanation}.\qed
\end{proof}
See Fig.~\ref{contrastive_singletons_pairs_appendix} for an illustration of contrastive singletons and pairs.

\section{Minimum Hitting Set (MHS) and Minimum Vertex Cover (MVC)}
\label{mhs-defintion-appendix}

\mysubsection{Minimum Hitting Set (MHS).} Given a collection
$\mathbb{S}$ of sets from a universe U, a hitting set $h$ for
$\mathbb{S}$ is a set such that
$\forall S \in \mathbb{S}, h\cap S \neq \emptyset$. A hitting set $h$
is said to be \emph{minimal} if none of its subsets is a hitting set,
and \emph{minimum} when it has the smallest possible cardinality among
all hitting sets.

\mysubsection{K Minimum Hitting Set (K-MHS).} K-MHS denotes the same
problem as MHS, but when sets are at the size of at most $k$. This is
a re-formalization of the minimum vertex cover (MVC) problem on a
k-hyper-graph, where sets are treated as edges and elements in sets are
treated as vertices. This implies that a 2-MHS problem is simply the
classic \emph{minimum vertex cover (MVC)} problem.

\section{Extending \tlb beyond Contrastive Pairs}
\label{extending-tlb-appendix}
The lower bounding thread, \tlb, (see Sec.~\ref{sec:approximations})
is used for computing a lower bound on the size of the minimum explanation. This
is done by computing contrastive singletons and contrastive pairs, and
using them in
calculating an under approximation for the MHS of all contrastive
examples. This approach can be extended to use contrastive examples of
larger sizes ($k=3,4,\ldots$). Finding these contrastive examples may
improve the approximation of the global minimum, but may also
render the approach less scalable. In the worst case, finding all sets of size $k$ requires
performing $O(\binom{m}{k})$ queries to the underlying
verifier. Since the search space becomes larger as $k$ increases, each
query may become more expensive as well. In case of  a large feature
spaces, if we are interested in performing an approximation via a
greedy algorithm, the quality of the approximation also deteriorates as $k$
increases. The general K-MHS problem has a polynomial $k$-approximating
algorithm, and this bound was shown to be tight for all
$k\geq3$~\cite{dinur2003new}. Theoretically, if $\tlb$ continues
finding contrastive examples of larger sizes, $k$ in the final step is
the minimum $k$ on which the MHS of all contrastive examples of size
$k$ and less are equal to the minimum explanation. The full, exact
approximation can be summarized by the following formula:
\begin{equation}
	\begin{aligned}
		\lowerb& = |\singletons|+|\mvc|\\
		&\leq 
		\sum_{\kisoneappendix}^{\kismaxkappendix}|\kmhskcxp| \\
		&=\mhscxp=\minimumexplanation 			
	\end{aligned}
\end{equation}
where k-Cxps denotes all contrastive examples of size $k$, and maxk denotes the size of $k$ in the final iteration. 


\section{Minimal Bundle Explanations}
\label{minimal-bundle-explanation-appendix}
Let $v=(v_1,...v_m)\in \mathbb{F}$ be an input with classification
$N(v)=c$, and let $B$ be the set of all bundles.  The definition
of a \emph{minimal bundle explanation} $E_B\subseteq B$ for instance
$(v,c)$ is:

\begin{equation}
	\begin{aligned}
		&(\forall(x\in \mathbb{F}).[\wedge_{i\in
			\cup E_B}(x_{i}=v_{i})\to(N(x)=c)]) \wedge\\
		&(\forall(j_B\in E_B).[  \exists(y\in \mathbb{F}).[\wedge_{i\in \cup E_B\setminus j_B}(y_{i}=v_{i})\wedge(N(y)\neq c)])
	\end{aligned}
\end{equation}

A \emph{minimum bundle explanation} is a minimal bundle explanation of minimum size.
We define a \emph{contrastive bundle example} $C_B$ for input $(v,c)$ and the set of all bundles $B$, as a subset of bundles $C_B\subseteq B$ such that:

\begin{equation}
	\label{eq:contrastive_bundle_example}
	\exists(x\in \mathbb{F}).[\wedge_{i\in \cup B\setminus C_B}(x_{i}=v_{i})\wedge(N(x)\neq c)]
\end{equation}

A \emph{contrastive bundle singleton} is defined as a contrastive bundle example of size 1, and a \emph{contrastive bundle pair} as a contrastive bundle example of size 2 (which does not contain contrastive bundle singletons).

\mysubsection{Minimum Hitting Set of Bundles (MHSB).}  We use the
common definition for a \emph{minimum hitting set of bundles
	(MHSB)}~\cite{angel2009minimum}, which is as follows. Given
a set of $n$ elements $\mathcal{E}=\{e_1,e_2,...,e_n\}$, each element
$e_i$ ($i=1,...n$) has a non-negative cost $c_i$. A bundle $b$ is a
subset of $\mathcal{E}$. We are also given a collection
$\mathcal{S}=\{S_1,S_2,...,S_m\}$ of $m$ sets of bundles. More
precisely, each set $S_j$ ($j=1,...,m$) is composed of $g(j)$ distinct
bundles $b_j^1 b_j^2, ..., b_j^{g(j)}$. A solution to MHSB is a subset
$\mathcal{E}' \subseteq \mathcal{E}$ such that for every
$S_j\in \mathcal{S}$, at least one bundle is covered, i.e,
$b_{j}^{l}\subseteq \mathcal{E}'$ for some $l\in
\{1,2,...,g(j)\}$. The total cost of the solution, denoted by
$C(\mathcal{E}')$, is $\sum_{\{i\ |\ e_i \in e'\}}c_i$ (a cost of an
element appearing in several bundles is counted once). The objective
is to find a solution with minimum total cost.

Notice that this is a more general definition than our case, where each element (feature) inside a bundle has a cost of 1, meaning that the cost of each bundle is the number of features it contains. The proven dual MHS relationship between explanations and contrastive examples~\cite{ignatiev2020contrastive} can thus be trivially expanded to include bundles ---i.e., that the MHSB of all contrastive bundle examples constitutes as the minimum bundle explanation and that the MHSB of all bundle explanations constitutes as a minimum bundle contrastive example.


\section{Approximating Bundle Explanations}\label{sec:bundle_approximation_appendix}

\mysubsection{Method 1: Approximating Minimum Bundle Explanations.}
We first discuss how to obtain a provable approximation for the minimum
bundle explanation. This is a trivial extension of our method
suggested in Sec.~\ref{sec:approximations}. The minimal bundle
explanation found by \tub is clearly an upper bound for the minimum
bundle explanation. Regarding the lower bound computed by \tlb, the
union of all bundle singletons is a lower bound, since every
contrastive bundle singleton is contained in the minimum bundle
explanation. Moreover, the minimum \emph{weighted} vertex cover of all
contrastive bundle pairs (where weights indicate the number of
features in each bundle) constitutes a lower bound for the MHSB of
all contrastive bundle examples. In cases of large feature spaces, a
2-approximating linear greedy algorithm can be used for the minimum
weighted vertex cover~\cite{bar1981linear}. Overall, the following
lower bound can be calculated:
\begin{equation}
	\label{appendix_bundle_approx_1_option}
	\begin{aligned}
		\lowerbbundle &= |\allbundlesingletons|+|\mwvcbpairs| \\
		&\leq\ |\mhsbbcxps| = \minimumbundle
	\end{aligned}
\end{equation}
where \lowerbbundle denotes the lower bound that is calculated for our evaluation, BSingletons denotes the set of all contrastive bundle singletons, Bpairs denotes the set of all contrastive bundle pairs, MWVC is the minimum weighted vertex cover, \bcxps denotes the set of all contrastive bundle examples, and \minimumbundle is the minimum bundle explanation.

\mysubsection{Method 2: Approximating Minimum (Non-Bundle) Explanations.}
The second approach is to calculate an approximation ratio with respect to a regular, non-bundle
minimum explanation. The minimal bundle explanation found by \tub is
an upper bound for the minimum non-bundle explanation, since the minimal bundle explanation is also an explanation by itself (Theorem~\ref{eq:bundle_explanation}). For obtaining the lower bound,
we can analyze contrastive bundle
examples found by \tlb for extracting contrastive non-bundle
examples, and thus enabling the computation of an under-approximation to the MHS of all
contrastive examples (an illustration is also depicted on Fig. \ref{fig:bundle_approximation_example}).

For example, it is straightforward that every contrastive bundle singleton is a contasrive example by itself, and thus at least one of the bundle's
elements is contained in a minimum explanation. Likewise, for every
contrastive bundle pair $(b_1,b_2)$ there exist at least two
subsets, $s_1\subseteq\ b_1$ and $s_2\subseteq b_2$, such that
$s_1\cup s_2$ is a contrastive example. This means that a regular
minimum (unweighted) vertex cover can be calculated by the following
approximation (used for our evaluation):
\begin{equation}
	\label{appendix_bundle_approx_2_option}
	\begin{aligned}
		|\bsingletons|+|\mvcbundle|\leq \minimumexplanation
	\end{aligned}
\end{equation}

An additional, optional approach for tightening the bound even more is to search for
contrastive examples of features within each bundle.  This can be done
for every contrastive bundle singleton, by calculating the MHS of all
contrastive examples within that certain bundle. Since bundles
typically consist of small feature sets, it may be computationally
feasible to compute the MHS of all features within each bundle. If
not, the procedure that we suggested in Sec.~\ref{sec:approximations}
can be repeated for each bundle. More precisely, we propose to iterate
on all features and pairs in each bundle, to find all contrastive
singletons and pairs within that bundle, and then to calculate a lower
bound by solving an unweighted vertex cover problem. Further, we can
perform a similar process (of either calculating the MHS or performing
the lower bound procedure from Sec.~\ref{sec:approximations}) on the
union of all contrastive bundle pairs, as well. Notice that by doing
so, we do not need to iterate on the entire set of all pairs, since
features that are within the same contrastive bundle are necessarily
not contrastive pairs (because otherwise, that bundle would be a
contrastive bundle singleton).  Thus, we can arrive at the following
bound:
\begin{equation}
	\label{appendix_bundle_approx_3_option}
	\begin{aligned}
		|\bsingletons|&+|\mvcbundle| \\
		&\leq
		|\sum_{\bsingletonsinstance \subseteq \bsingletons}\lowerb(\bsingletonsinstance)|+|\lowerb(\cup \bpairs)| \\
		&\leq
		\sum_{\bsingletonsinstance \subseteq \bsingletons}|\mhsbsingletons|+|\mhsbpairs| \\
		&\leq \minimumexplanation
	\end{aligned}
\end{equation}
\vspace{3cm}
\section{Additional Evaluation}\label{appendix-evaluations}

\mysubsection{Fashion-MNIST Evaluation.}
Figs.~\ref{fig:final_basic_results_appendix_fashion} and
\ref{fig:fashion_upper_lower} depict the results of our evaluation using
the Fashion-MNIST network.

\begin{figure}[h]
	\centering
	\subfloat[Average portion of features verified to be included/excluded from explanation]{\includegraphics[width=0.5\textwidth]{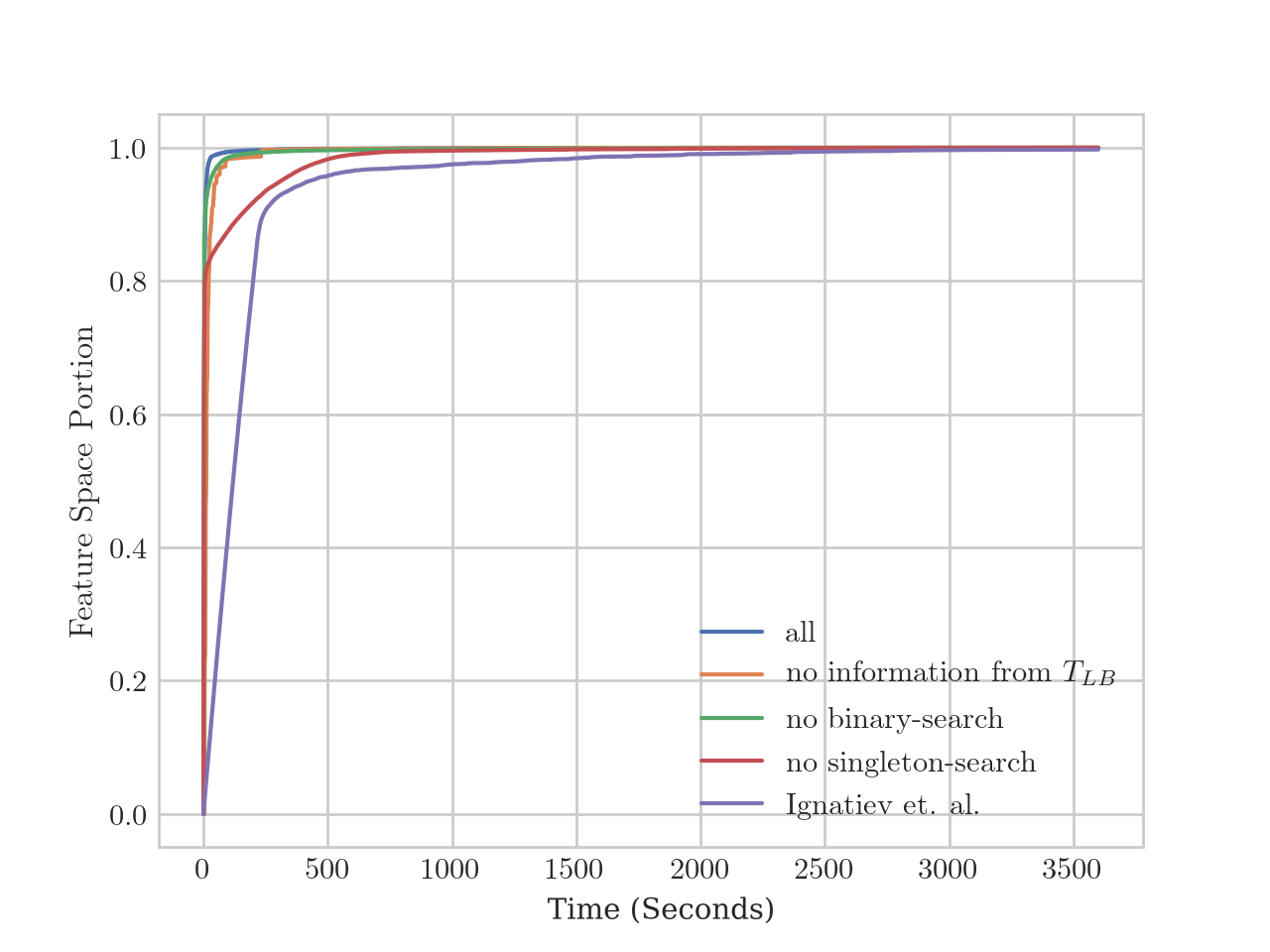}\label{fig:appendix_fashion_1}}
	\hfill
	\subfloat[Average size of explanation]{\includegraphics[width=0.5\textwidth]{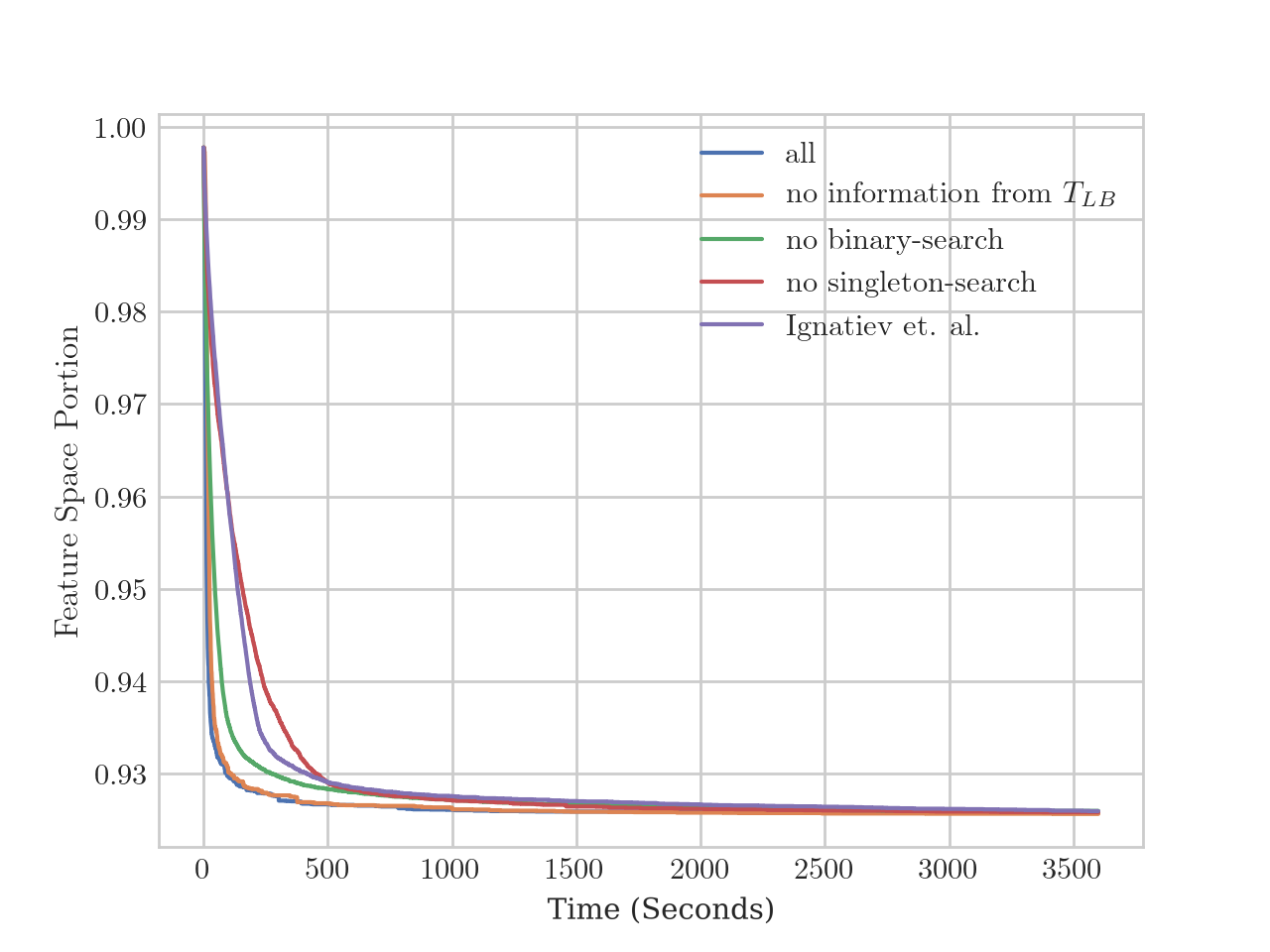}\label{fig:appendix_fashion_2}}
	\caption{Our full and ablated results compared to the state-of-the-art for finding minimal explanations over the Fashion-MNIST dataset}
	\label{fig:final_basic_results_appendix_fashion}
\end{figure}

\begin{figure}[h]
	\vspace{-0.3cm}
	\centering
	\includegraphics[width=0.5\textwidth]{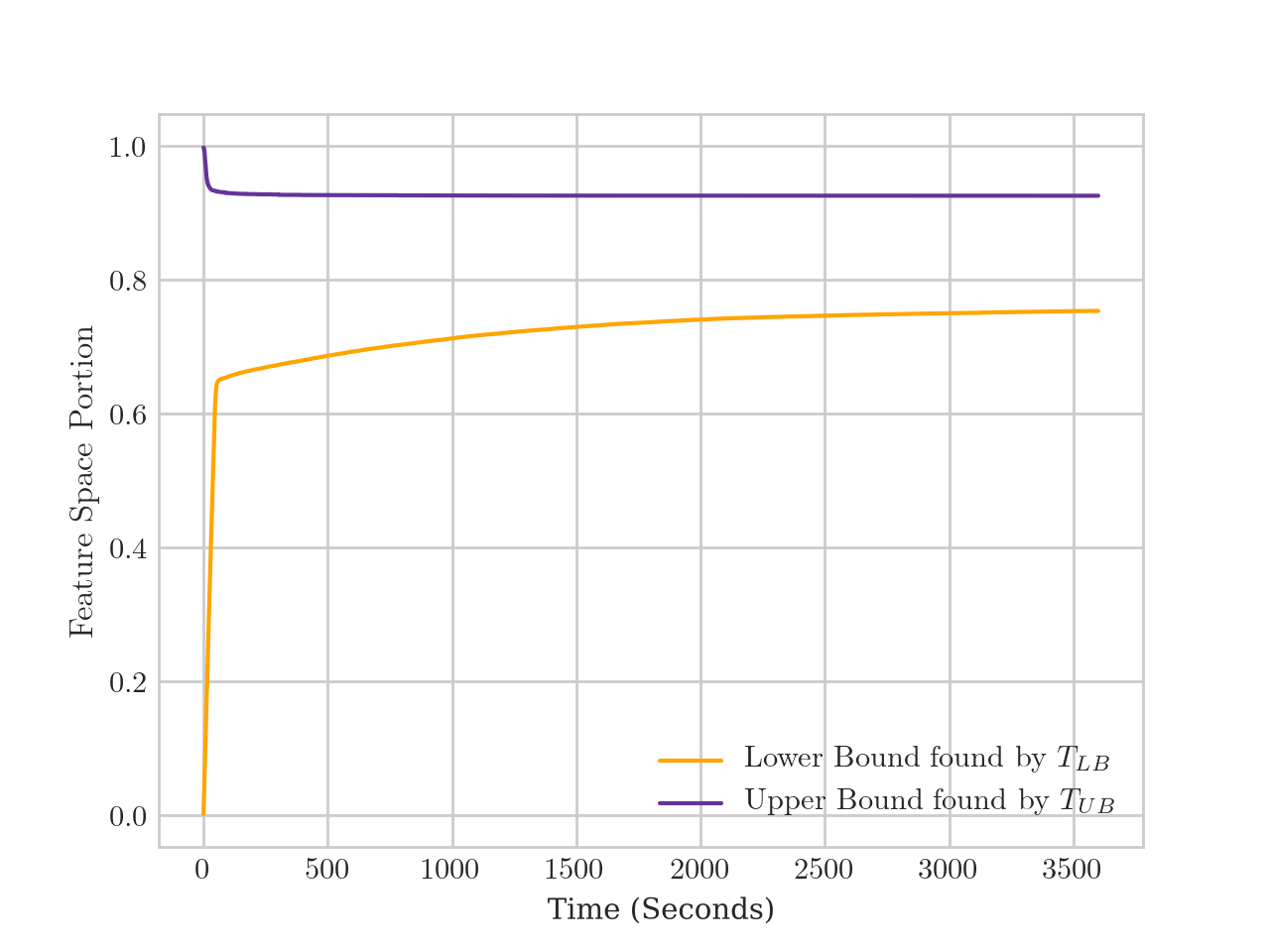}
	\caption{Average approximation of \emph{minimum} explanation over time on the Fashion-MNIST dataset}
	\label{fig:fashion_upper_lower}
\end{figure}
\mysubsection{Minimal Bundle Explanation Experiments.}
Table~\ref{tab:analysis} summarizes the results of our bundle
explanation experiments.


\begin{center}
	\begin{table}
		\caption{Bundle and non-bundle implementation results,
			evaluated using the following criteria: (\romannumeral 1)
			the proportion of completed tasks within a one hour time
			limit; (\romannumeral 2) the average run-time per task, in
			seconds; (\romannumeral 3) the average final approximation
			ratio found of the minimum bundle explanation:
			$E_{\text{MB}\xspace}$ (lower bound obtained per
			Eq.~\ref{appendix_bundle_approx_2_option}); and
			(\romannumeral 4) the average final approximation found of
			the minimum explanation: $E_{M}$ (lower bound obtained per Eq.~\ref{appendix_bundle_approx_1_option}).}
		\vspace{0.5cm}
		\label{tab:analysis}
		\centering
		\scalebox{0.75}{
			\begin{tabular}{ |c|c|c|c|c| }
				\hline
				\textbf{Dataset} & \textbf{Criteria} & \textbf{Bundle Explanation} &  \textbf{Non-Bundle Explanation} \\
				\hline
				\multirow{4}{*}{MNIST} & Completed Portion & 100\% & 16\%  \\
				& Run-time (Seconds) & 115.3 &  3264.5 \\
				& $E_{MB}$ Approx. & 1.49 & -  \\
				& $E_{M}$ Approx. & 5.13 & 1.61 \\
				\hline
				\multirow{4}{*}{Fashion-MNIST} & Completed Portion & 100\% & 100\%  \\
				& Run-time (Seconds) & 119.6 & 140.82  \\
				& $E_{MB}$ Approx. & 1.15 & - \\
				& $E_{M}$ Approx. & 1.62 & 1.23  \\
				\hline
			\end{tabular}
		}
	\end{table}
\end{center} 
\FloatBarrier
\section{Additional Figures}
We present here additional illustrations that demonstrate contrastive
examples (specifically, contrastive singletons and pairs)
(Fig.~\ref{contrastive_singletons_pairs_appendix}), an illustration of the binary search heuristic~\ref{fig:binary-search} and the local singleton search heuristic~\ref{fig:local-singelton-search} for expediting the search for minimal explanations, and minimal bundle
explanation approximations
(Fig.~\ref{fig:bundle_approximation_example}). Additionally, we attach
an illustration of one-pixel attacks, which can be regarded as
contrastive singletons (Fig.~\ref{fig:one-pixel-attack}).

\begin{figure}
	\centering
	\subfloat[Contrastive singletons]{\includegraphics[width=0.47\textwidth]{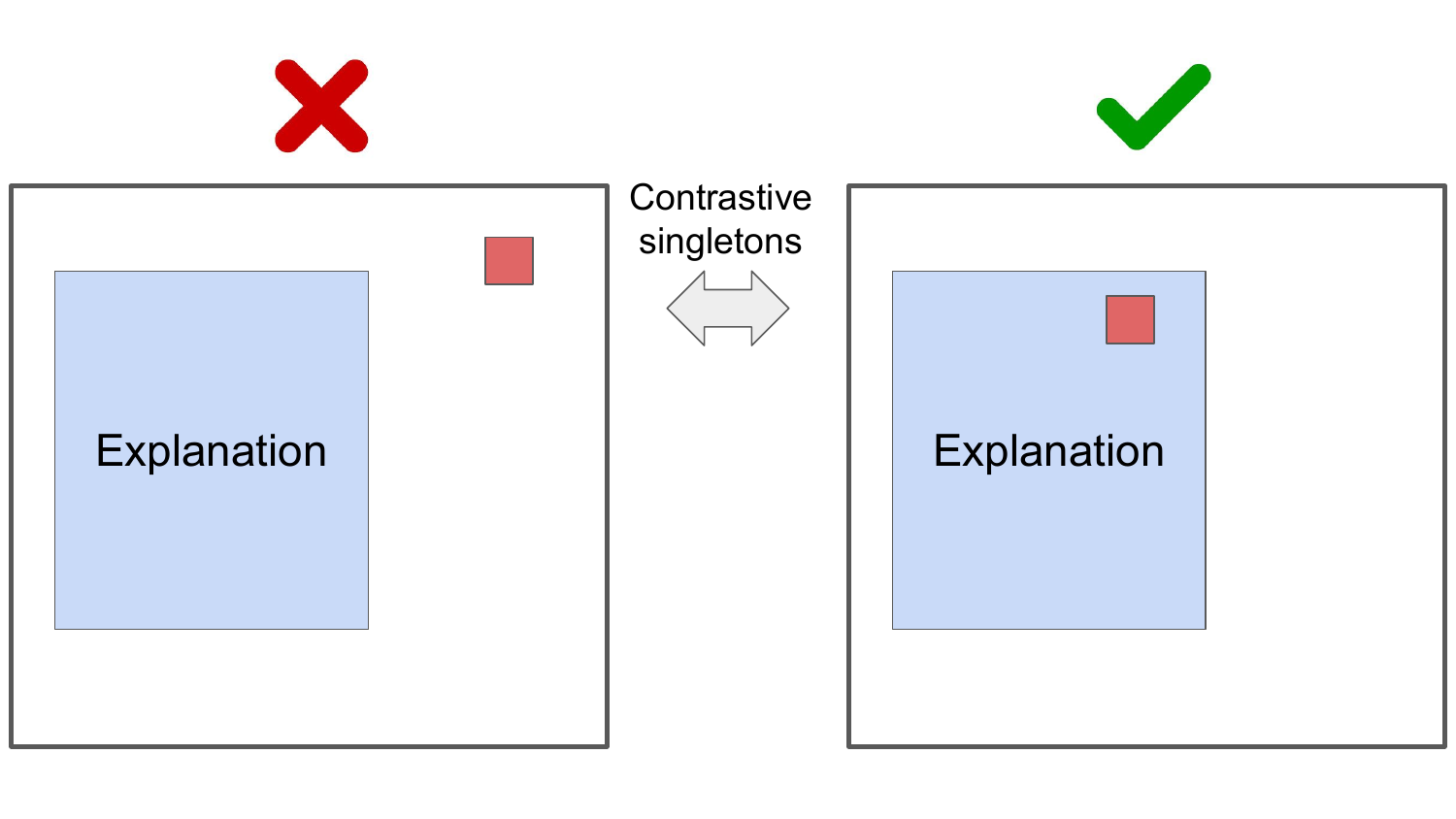}\label{fig:contrastive_singeltons_illustration}}
	\hfill
	\subfloat[Contrastive pairs]{\includegraphics[width=0.47\textwidth]{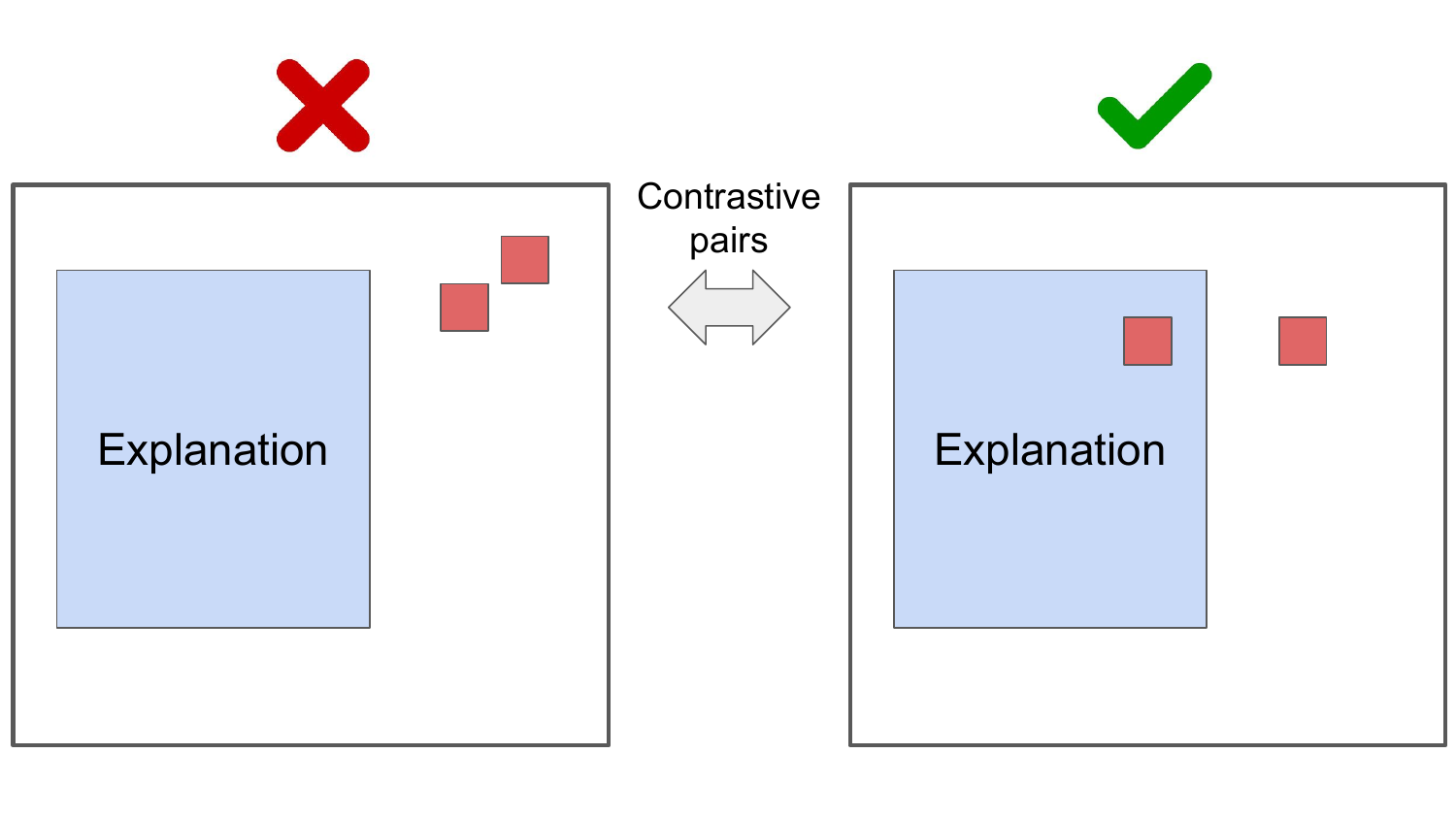}\label{fig:contrastive_pairs_illustration}}
	\caption{An illustration of contrastive singletons and contrastive pairs. Every singleton is contained in every explanation, and every pair holds at least one feature in every explanation.}
	\label{contrastive_singletons_pairs_appendix}
\end{figure}
\begin{figure}
	\centering
	\includegraphics[width=0.6\textwidth]{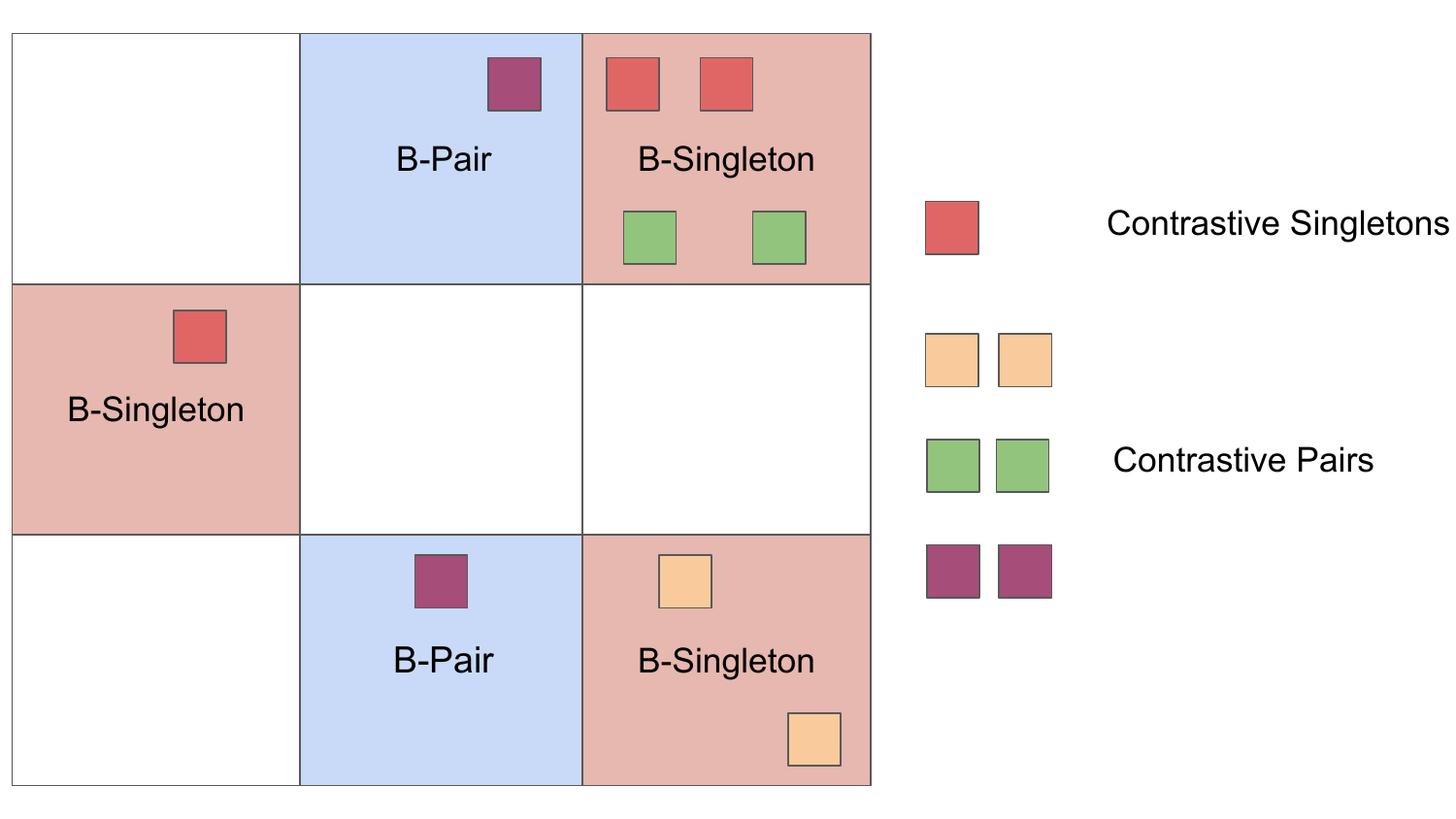}
	\caption{An illustration of a bundle explanation. B-Singletons and
		B-Pairs represent contrastive bundle singletons and
		pairs. A lower bound can be computed either in terms
		of bundles, or in terms of the features that
		comprise them.}
	\label{fig:bundle_approximation_example}
\end{figure}
	\begin{figure}
		\begin{center}
			\includegraphics[width=0.7\linewidth]{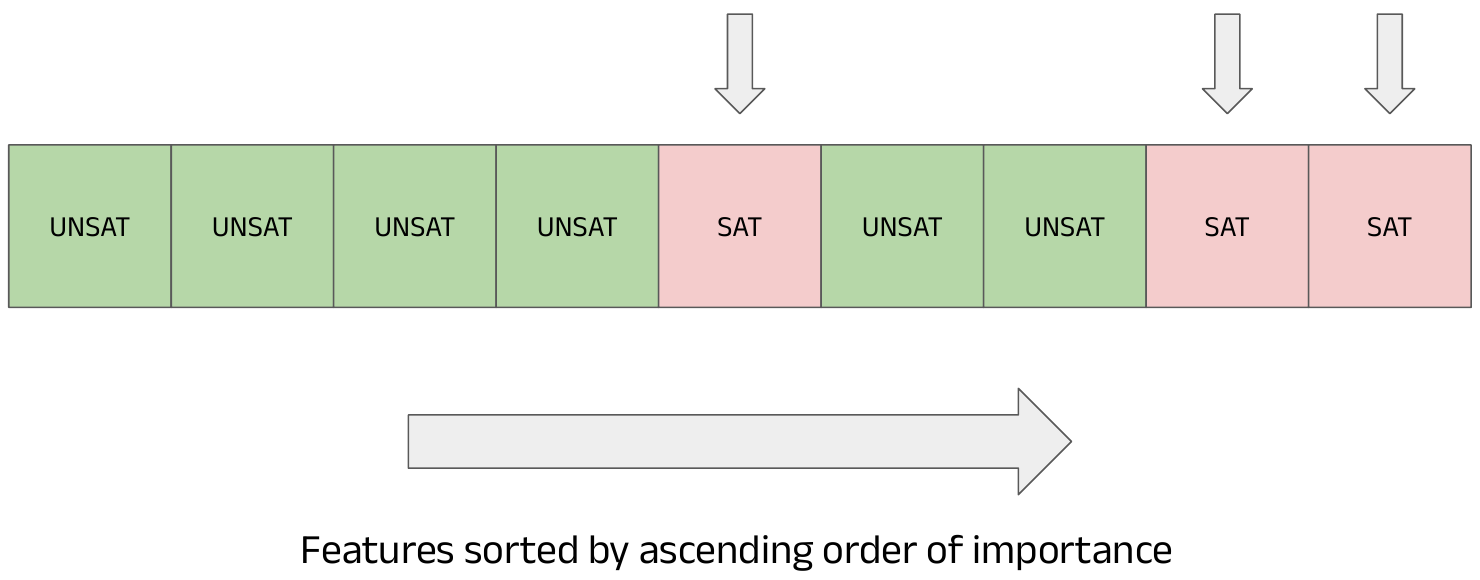}
			\caption{Features are sorted in increasing order of
				importance. \sat{} queries indicate features included in the
				explanation. We perform a sequence of binary searches to
				identify these \sat{} queries.
			}
			\label{fig:binary-search}
		\end{center}
	\end{figure}

 	\begin{figure}[!htp]
		\begin{center}
			\includegraphics[width=0.7\linewidth]{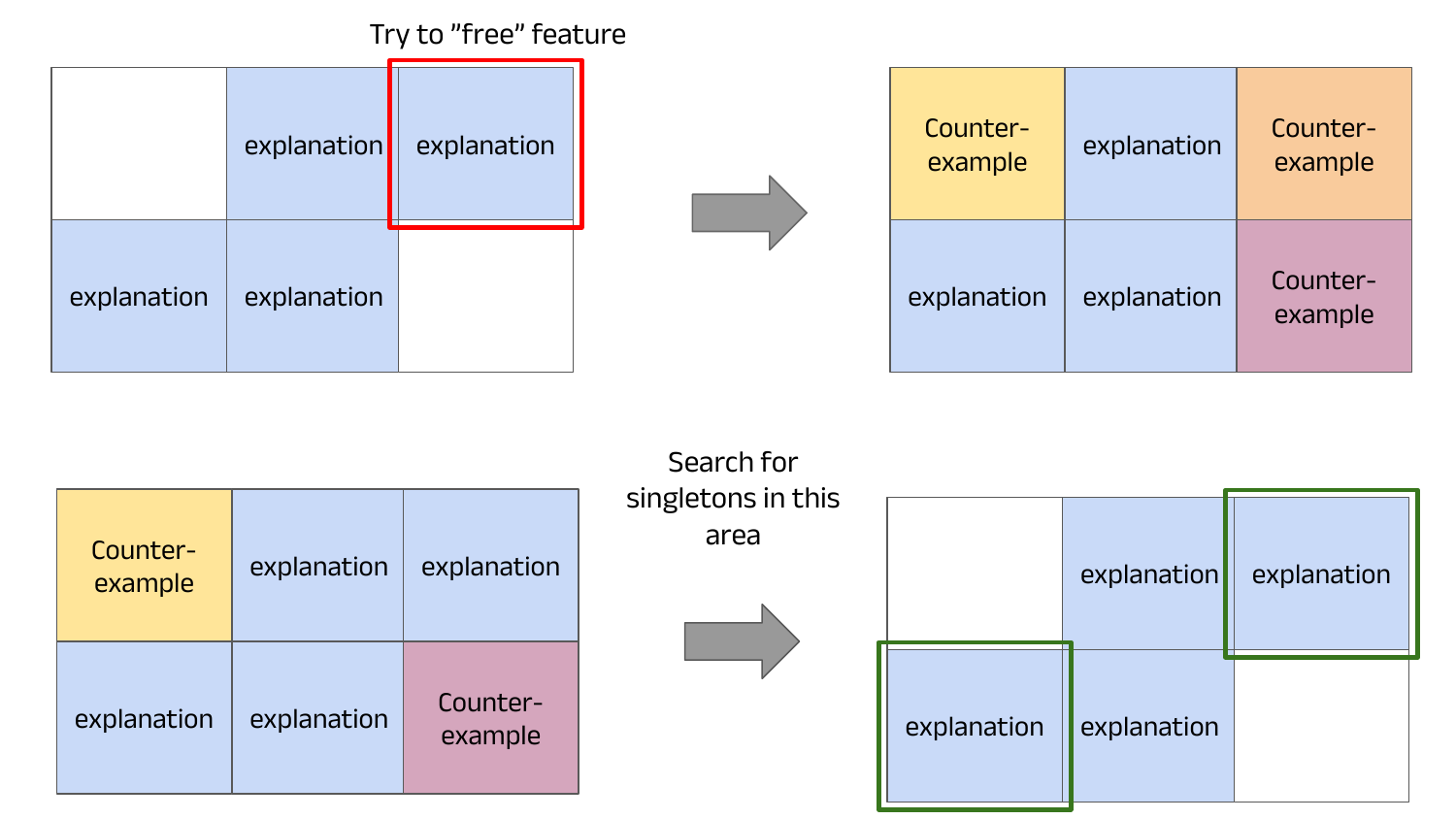}
			\caption{An illustration of the local-singleton search. Elements in blue
				represent features that are currently in the
				explanation (not ``freed''), and are thus set to the input
				value. We try to ``free'' a feature, find a
				counterexample, and search for \emph{local
					singletons} in the nearby area of that counter-example. These
				features are bound to be part of the
				explanation, hence reducing the number of
				features that remained for verifying.}
			\label{fig:local-singelton-search}
		\end{center}
	\end{figure}
\begin{figure}
	\begin{center}
		\includegraphics[width=0.5\linewidth]{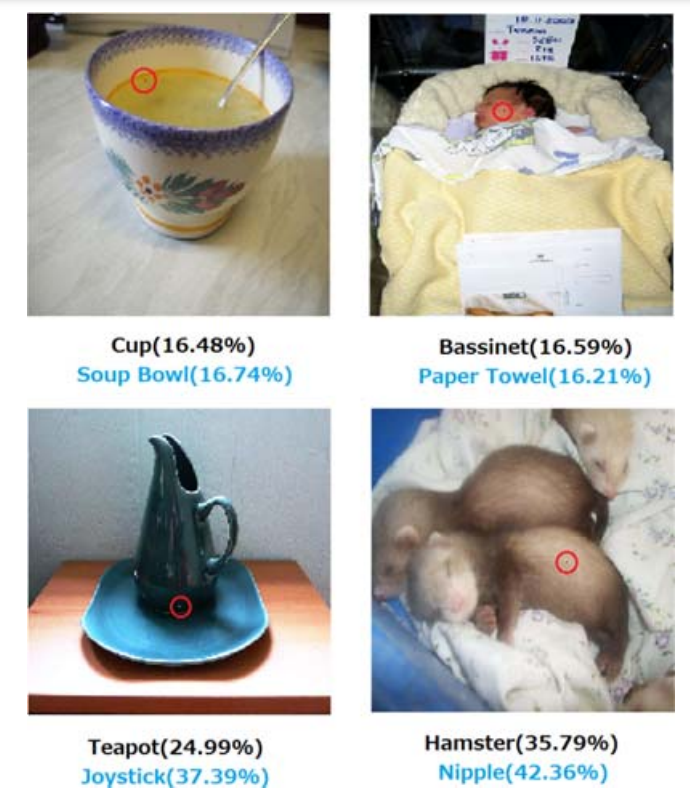}
		\caption{Examples of one-pixel-attacks on the Imagenet
			database, borrowed from~\cite{su2019one}. These represent
			\emph{contrastive singletons}.}
		\label{fig:one-pixel-attack}
	\end{center}
\end{figure}

\section{Additional Pseudo-Codes}

We present here additional pseudo-codes, specifically of the binary search heuristic~\ref{alg:upper-thread-binary-search} for expediting the search for minimal explanations.
	\begin{algorithm}
		
		\algnewcommand\algorithmicforeach{\textbf{for each}}
		\algdef{S}[FOR]{ForEach}[1]{\algorithmicforeach\ #1\ \algorithmicdo}
		\caption{\tub using binary-search}\label{alg:upper-thread-binary-search}
		\begin{algorithmic}[1]
			\State{Use a heuristic model to sort $F$ by ascending relevance}
			\State $L=0$
			\State $R=|F|-1$
			\While {$L\leq |F|-1$}
			\While {$L\leq R$} \Comment{The inner loop is a single binary search}
			\State Mid $\gets \frac{L-R}{2}$
			\State \explanation$\gets$F$\setminus$Free
			\If{\verifybinary is \unsat{}}
			\State \free $\gets$ \free $\cup$ \{L,L$+1$,...,Mid\}
			\State \upperb $\gets$ \upperb$- |\{L,L$+1$,...,$Mid$\}|$
			\State $L$ $\gets$ Mid$+1$
			\Else
			\State $R$ $\gets$ Mid$-1$
			\EndIf
			\EndWhile
			\State $L$ $\gets$ $L+1$
			\State $R$ $\gets |F|-1$
			\EndWhile
		\end{algorithmic}
	\end{algorithm}

\end{document}